\pdfoutput=1
\documentclass[]{ant_tech_report}

\usepackage[utf8]{inputenc}
\usepackage{fontawesome5}
\usepackage{inconsolata}
\usepackage{url}
\usepackage{wrapfig}
\usepackage{adjustbox}
\usepackage{mathtools}
\usepackage{algorithm}
\usepackage{algpseudocode}
\usepackage{float}

\usepackage[utf8]{inputenc}
\usepackage[T1]{fontenc}
\usepackage{booktabs}
\usepackage{longtable}
\usepackage{amsfonts}
\usepackage{nicefrac}
\usepackage{microtype}
\usepackage{xcolor}
\usepackage{graphicx}
\usepackage{tikz}
\usepackage{subcaption}
\usepackage{xspace}
\usepackage{enumitem}
\usepackage{amsmath}
\usepackage{amssymb}
\usepackage{pifont}
\usepackage{listings}
\usepackage{tcolorbox}
\tcbuselibrary{breakable,listings}
\usepackage{textcomp}
\usepackage{upquote}
\usepackage{multirow}
\usepackage{colortbl}
\usepackage{makecell}
\usetikzlibrary{arrows.meta,positioning,fit}

\usepackage{array}
\newcolumntype{L}[1]{>{\raggedright\arraybackslash}p{#1}}
\newcolumntype{C}[1]{>{\centering\arraybackslash}p{#1}}
\newcolumntype{R}[1]{>{\raggedleft\arraybackslash}p{#1}}

\definecolor{forgeblue}{HTML}{E3F2FD}
\definecolor{forgelight}{HTML}{F5F7FA}
\definecolor{forgepurple}{HTML}{7E57C2}
\definecolor{forgegreen}{HTML}{43A047}
\definecolor{forgered}{HTML}{FB8C00}
\definecolor{forgegray}{HTML}{E0E0E0}
\definecolor{forgeprimary}{HTML}{1565C0}
\definecolor{forgeprimarydark}{HTML}{0D47A1}
\definecolor{forgecyan}{HTML}{00ACC1}
\definecolor{oursrowcolor}{HTML}{E3F2FD}
\definecolor{baselinerowcolor}{HTML}{F5F7FA}

\setlist[itemize]{leftmargin=*, topsep=2pt, itemsep=1pt}
\setlist[enumerate]{leftmargin=*, topsep=2pt, itemsep=1pt}


\newcounter{insight}

\newtcolorbox[auto counter]{promptbox}[1][]{
    top=15pt,
    bottom=15pt,
    left=20pt,
    right=20pt,
    colback=gray!5,
    colframe=black,
    fonttitle=\bfseries,
    coltitle=white,
    title=Prompt~\thetcbcounter: #1,
    breakable,
    fontupper=\ttfamily\small,
}

\newtcblisting[auto counter]{promptlistingbox}[1][]{
    top=15pt,
    bottom=15pt,
    left=20pt,
    right=20pt,
    colback=gray!5,
    colframe=black,
    fonttitle=\bfseries,
    coltitle=white,
    title=Prompt~\thetcbcounter: #1,
    breakable,
    listing only,
    listing options={
        basicstyle=\ttfamily\footnotesize,
        breaklines=true,
        breakatwhitespace=false,
        columns=fullflexible,
        keepspaces=true,
        showstringspaces=false,
        upquote=true,
        literate=*{_}{\_}{1}
    }
}

\AtBeginDocument{%
  \setlength{\abovedisplayskip}{5pt plus 1pt minus 2pt}%
  \setlength{\belowdisplayskip}{5pt plus 1pt minus 2pt}%
  \setlength{\abovedisplayshortskip}{3pt plus 1pt}%
  \setlength{\belowdisplayshortskip}{3pt plus 1pt}%
}
\usepackage{enumitem}
\setlist[itemize]{nosep}

\crefname{proposition}{Proposition}{Propositions}
\Crefname{proposition}{Proposition}{Propositions}
\crefname{lemma}{Lemma}{Lemmas}
\Crefname{lemma}{Lemma}{Lemmas}
\crefname{algorithm}{Algorithm}{Algorithms}
\Crefname{algorithm}{Algorithm}{Algorithms}

\newcommand{\method}{BPM}
\newcommand{\methodfull}{Byte-Prefix Marginalization}
\newcommand{\Vt}{\mathcal{V}^{\mathrm{t}}}
\newcommand{\Vs}{\mathcal{V}^{\mathrm{s}}}
\newcommand{\rmt}{\mathrm{t}}
\newcommand{\rms}{\mathrm{s}}
\newcommand{\bytes}{\mathrm{bytes}}

\definecolor{blockbg}{RGB}{243,247,240}
\definecolor{teacherbg}{RGB}{242,243,245}
\definecolor{bestbg}{RGB}{224,239,215}
\definecolor{sndbg}{RGB}{242,248,237}
\newcommand{\best}[1]{\cellcolor{bestbg}\textcolor{black!88}{\textbf{#1}}}
\newcommand{\snd}[1]{\cellcolor{sndbg}\underline{#1}}

\definecolor{promptframe}{HTML}{3E7A2E}
\definecolor{promptback}{HTML}{F4F9F0}
\renewtcblisting{promptbox}[1]{%
  breakable, enhanced, arc=2pt, boxrule=0.5pt,
  colframe=promptframe, colback=promptback,
  coltitle=white, fonttitle=\bfseries\small,
  title={#1}, top=4pt, bottom=4pt, left=6pt, right=6pt,
  listing only,
  listing options={
    basicstyle=\ttfamily\small, breaklines=true,
    breakatwhitespace=true, keepspaces=true, columns=fullflexible,
    breakindent=0pt, postbreak={},
  },
}

\definecolor{caseframe}{HTML}{55575A}
\definecolor{caseback}{HTML}{F7F8F7}
\newtcblisting{casebox}[1]{%
  enhanced, arc=2pt, boxrule=0.5pt,
  colframe=caseframe, colback=caseback,
  coltitle=white, fonttitle=\bfseries\small,
  title={#1}, top=4pt, bottom=4pt, left=6pt, right=6pt,
  listing only,
  listing options={
    basicstyle=\ttfamily\footnotesize, breaklines=true,
    keepspaces=true, columns=fullflexible,
    breakindent=0pt, postbreak={},
    aboveskip=0pt, belowskip=0pt,
  },
}

\title{Cross-Tokenizer On-Policy Distillation\\ via Byte-Prefix Marginalization}

\author[1,2]{Hao Wang}
\author[2,\dag]{Kun Yuan}
\author[2,3]{Wenlin Zhong}
\author[2]{Minglei Zhang}
\makeatletter
\g@addto@macro\authorlist{\\ \authorformat[4]{Han Xiao}}
\makeatother
\author[2]{Ming Sun}
\author[1,*]{Honggang Qi}

\affiliation[1]{University of Chinese Academy of Sciences}
\affiliation[2]{KwaiKAT Team}
\affiliation[3]{Zhejiang University}
\makeatletter
\g@addto@macro\affiliationlist{\\ \affiliationformat[4]{The Chinese University of Hong Kong}}
\makeatother

\renewcommand{\contributionlist}{\email{wanghao251@mails.ucas.ac.cn} \quad \email{yuankun03@kuaishou.com} \quad \email{hgqi@ucas.ac.cn}}
\newcommand{\authorfootnotes}{$^\dag$Project lead. \quad $^*$Corresponding author.}
\fancypagestyle{firststyle}{
    \fancyhf{}
    \fancyhead[L]{
        \raisebox{-12.5mm}[0pt][0pt]{%
            \vbox{\vskip 3mm\hbox{\includegraphics[height=10mm]{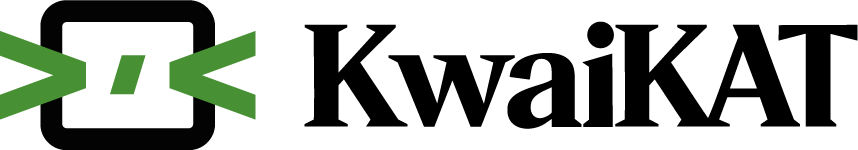}}}%
        }%
    }
    \fancyfoot[L]{\raisebox{1.15\baselineskip}[0pt][0pt]{\parbox[b]{\textwidth}{\rule{2in}{0.4pt}\\[3pt]\authorfootnotes}}}
    \fancyfoot[C]{\thepage}
}
\makeatletter
\renewcommand\checkdata[2][]{%
  \if@checkdataempty
    \g@addto@macro\checkdatalist{\checkdatabutton{#1}{#2}}%
    \global\@checkdataemptyfalse
  \else
    \g@addto@macro\checkdatalist{\hskip 2.4mm\relax\checkdatabutton{#1}{#2}}%
  \fi
}
\makeatother
\checkdata[\faRoute\hspace{4pt}Project Page]{https://bpm-opd.github.io/}
\checkdata[\faCubes\hspace{4pt}Models]{https://huggingface.co/K1zE/BPM}
\checkdata[\faDatabase\hspace{4pt}Datasets]{https://huggingface.co/datasets/K1zE/BPM}

\abstract{Open-weight language models from different families exhibit complementary capabilities, motivating their consolidation into a compact student through on-policy distillation (OPD). However, full-vocabulary OPD typically assumes a shared tokenizer, while existing cross-tokenizer methods may discard teacher probability mass or assign it to student tokens with unrelated content. We introduce \methodfull{} (\method{}), which re-expresses the teacher's next-token distribution over the student vocabulary in a shared byte space. Specifically, \method{} assigns each teacher token's probability to the longest student token whose byte representation is a prefix of the teacher token's bytes, aggregates mass mapped to the same student token, and places otherwise unmatched mass in an explicit residual category. This produces a vocabulary-complete, byte-aligned, and mass-preserving target for dense OPD. The target exactly recovers the teacher-induced byte-prefix marginal when the relevant prefix does not span multiple teacher tokens (a condition satisfied at more than $99\%$ of training positions) and uses a mass-preserving, chain-factorized lower bound otherwise. 
Across Qwen3-32B, GLM-Z1-9B-0414, and MiniMax-M2.7 as teachers, \method{} consistently outperforms current cross-tokenizer methods on six mathematics and programming benchmarks, improving six-benchmark avg@8 by $3.7$--$6.6$ points over the strongest baselines.
}

\begin{document}
\maketitle

\vspace*{\stretch{1}}
\begin{center}
\noindent\begin{minipage}{\textwidth}
\centering
\includegraphics[width=0.98\textwidth]{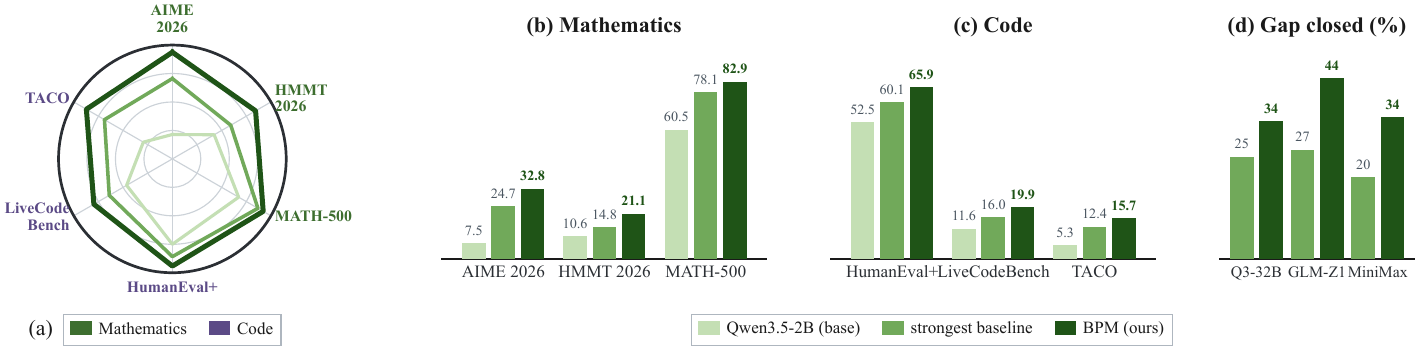}
{\captionsetup{font=small}
\captionof{figure}{Performance averaged over the three teacher--student pairs (\Cref{tab:main}). \textbf{(a)}~avg@8 on all six benchmarks, with each axis scaled independently; \textbf{(b)}~mean mathematics avg@8; \textbf{(c)}~mean code avg@8; \textbf{(d)}~fraction of each pair's teacher--student avg@8 gap closed by \method{} and its strongest baseline. \method{} leads all budget-matched on-policy baselines and the separately trained SeqKD reference.}
\label{fig:teaser}}
\end{minipage}
\end{center}
\vspace*{\stretch{2}}

\newpage
\section{Introduction}
\label{sec:intro}

\looseness=-1 Open-weight language models~\citep{deepseekai2026v4,qwen2025qwen3,kimi2026k25,glm2026glm5,minimax2026m2series} have advanced rapidly, with different model families exhibiting complementary strengths in mathematical reasoning, complex problem solving, agentic coding, and tool use. In practical deployments, it is often desirable to consolidate these diverse capabilities into a single compact model. Beyond capability integration, such consolidation can improve inference efficiency by reducing computational and memory overhead and avoiding the need to deploy or route requests among multiple specialized models. This motivates transferring knowledge from heterogeneous teacher models into an efficient student model~\citep{jiang2023llmblender}.

\looseness=-1 Knowledge distillation provides a natural mechanism for this transfer by training the student to imitate the teacher's predictive behavior~\citep{hinton2015distilling,kim2016sequence,xu2024surveykd,guha2025openthoughts}. On-policy distillation (OPD) performs this transfer along student-generated trajectories, querying the teacher at the states the student actually visits~\citep{agarwal2024onpolicy,gu2024minillm,arora2023exposure,wang2026skill}. Its supervision ranges from sampled-token and top-$k$ approximations~\citep{tml2025onpolicy,gu2024minillm,li2025bild,anshumann2025sparsekd} to full-vocabulary matching of the teacher's complete next-token distribution~\citep{agarwal2024onpolicy,qwen2025qwen3}.

\looseness=-1 However, dense token-level supervision typically assumes that the teacher and student use the same tokenizer~\citep{sennrich2016bpe}. This assumption restricts teacher selection to models compatible with the student's tokenization scheme, potentially excluding a stronger teacher for the target task~\citep{rust2021tokenizer}. Extending full-vocabulary OPD to heterogeneous tokenizers therefore requires the transformed target to satisfy three properties. \textbf{(D1) Vocabulary Completeness:} the target is well-defined for every token in the student vocabulary. \textbf{(D2) Byte-level Alignment:} the teacher's probability mass is mapped to student tokens according to the correspondence between their byte-level representations. \textbf{(D3) Mass Preservation:} the transformation preserves the teacher's total probability mass, ensuring that the resulting target distribution, including any explicitly represented residual mass, sums to one.

\looseness=-1 Existing cross-tokenizer distillation methods~\citep{wan2024fusellm,boizard2024uld,zhang2024dskd,cui2025multilevelot,minixhofer2025alm,patino2025gold,sun2026simct} make teacher and student distributions comparable by relaxing at least one of these requirements (\Cref{sec:related}). Candidate-set methods such as SimCT~\citep{sun2026simct} leave unmatched teacher mass untransferred, violating \textbf{(D1) Vocabulary Completeness}; rank matching in ULD~\citep{boizard2024uld} and GOLD~\citep{patino2025gold} may assign teacher mass to student tokens with unrelated byte content, violating \textbf{(D2) Byte-level Alignment}. SeqKD~\citep{kim2016sequence} instead avoids distribution matching and provides no full-vocabulary next-token target. These baselines therefore either approximate token-level transfer or avoid token-level distribution matching entirely. At the byte level, however, the transformation admits an exact solution.

\looseness=-1 We therefore formulate teacher supervision in a shared byte space~\citep{clark2022canine,xue2022byt5,yu2023megabyte,pagnoni2024blt}: although the teacher and student may tokenize a response differently, they agree on its underlying byte sequence (\Cref{fig:overview}). \methodfull{} (\method{}) transforms the teacher's next-token distribution into a distribution over the student vocabulary. Specifically, the probability mass of each teacher token is assigned to the longest student token whose byte representation is a prefix of the teacher token's byte representation. The mass assigned to the same student token is then aggregated across teacher tokens, while any mass for which no such student-token prefix exists is allocated to an explicit residual category. This residual is essential for mass preservation: without it, unmatched teacher mass would be discarded and the transformed target could sum to less than one. The resulting target therefore satisfies D1/D2/D3 by construction. At more than $99\%$ of training positions, this target exactly recovers the teacher's byte-prefix marginal. The student is then trained on its own generated trajectories by minimizing a generalized Jensen--Shannon divergence with respect to the transformed target.

\looseness=-1 Our experiments further identify a failure mode of byte-exact distillation. At positions where the target consists solely of whitespace in code-generation data, byte-level alignment can transfer tokenizer-specific segmentation patterns rather than the teacher's predictive behavior. For one teacher--student pair, this spurious supervision reduces the code pass rate from $49\%$ to $9\%$ (\Cref{sec:collapse}). We address this issue with a precomputed mask that excludes all-whitespace target positions from the distillation objective, thereby preventing the observed collapse.

In summary, our contributions are twofold:
\begin{itemize}
    \item \textbf{Exact full-vocabulary supervision across tokenizers.}
    We introduce \method{}, which re-expresses a teacher's next-token distribution over the student vocabulary while satisfying D1-D3 by construction. We prove that the resulting target is exact whenever the byte prefixes do not span multiple teacher tokens, a condition that holds at more than $99\%$ of the training positions in our experiments (\Cref{prop:scatter}). At the remaining spanning positions, \method{} uses a mass-preserving, chain-factorized lower bound.
    
    \item \textbf{A systematic evaluation across heterogeneous tokenizers.}
    We independently distill three teachers, including Qwen3-32B, GLM-Z1-9B-0414, and MiniMax-M2.7, into the same thinking-mode student architecture, covering teacher--student pairs with varying tokenizer distances. All on-policy methods share a training budget; SeqKD follows its standard offline recipe with no smaller budget. The best \method{} variant outperforms the strongest baseline for every teacher-student pair, closing $33.5\%$, $43.9\%$, and $34.4\%$ of the corresponding teacher-student performance gaps.

\end{itemize}

\section{Related Work}
\label{sec:related}

\subsection{Distillation for Language Models}
\looseness=-1 Language-model distillation began with matching the teacher's output distribution position by position~\citep{hinton2015distilling}. SeqKD carried this recipe to sequence models~\citep{kim2016sequence}. Other lines transfer intermediate features~\citep{romero2015fitnets}, reasoning rationales~\citep{hsieh2023distilling,magister2023teaching}, or a teacher-refined training corpus~\citep{gu2025miniplm,xiao2025ui,xiao2026ui}. \method{} builds on two refinements of the per-position recipe. MiniLLM replaces the forward KL with a reverse KL~\citep{gu2024minillm}. The student then concentrates on the teacher modes it can represent. Generalized on-policy distillation trains on the student's own generations~\citep{agarwal2024onpolicy,lin2020imitkd,li2026rethinking}. Its divergence interpolates between the two KLs, and later objectives span f-divergences and skew or contrastive variants~\citep{wen2023fdistill,ko2024distillm,ko2025distillm2}. On-policy distillation has reached production. KAT-Coder-V2 consolidates five specialized coding experts into one deployable model~\citep{katteam2026katcoderv2}. Its successor extends the same pipeline to repository-scale agentic tasks~\citep{katteam2026katcoderv25}. The token-level objectives in this line still compare teacher and student predictions within a shared vocabulary.

\subsection{Cross-Tokenizer Distillation}
\looseness=-1 Cross-tokenizer methods recover comparability through alignment, transport, projection, or span-level objectives. FuseLLM hard-aligns the two token sequences by minimum edit distance~\citep{wan2024fusellm}. The per-position comparison becomes one approximate token-to-token map. The Universal Logit Distillation loss sorts each distribution by probability and matches it rank by rank as an optimal-transport surrogate. The match discards which token each ranked probability belongs to~\citep{boizard2024uld}. Multi-level optimal transport extends this rank-based transport across token and sequence levels~\citep{cui2025multilevelot}. The GOLD implementation adds a generalized Jensen--Shannon term to the sorted loss~\citep{patino2025gold}. The term covers only the vocabulary the two models share. Dual-space distillation learns a projection into a common hidden space. The output-side correspondence is then only as reliable as the learned projection~\citep{zhang2024dskd}. SimCT compares the two models on short continuations both tokenizers can express, and supervises only a restricted set of positions. The rest of the student vocabulary receives no signal~\citep{sun2026simct}. Approximate likelihood matching is the closest prior objective and also supervises across tokenizers on aligned spans~\citep{minixhofer2025alm}. Its target for each span is one scalar likelihood, and a scalar cannot place mass on a specific student token. Zero-shot tokenizer transfer instead re-equips a model with a new vocabulary~\citep{minixhofer2025zett}.

\paragraph{Inference-time vocabulary projection.}
GaC, EVA, and UNITE project ensemble distributions onto a shared vocabulary~\citep{yu2024gac,xu2024eva,yao2025unite}. These projections operate only at inference time, whereas \method{} constructs an explicit training distribution over the full student vocabulary.

\section{Method}
\label{sec:method}

\begin{figure*}[t]
\centering
\includegraphics[width=0.96\textwidth]{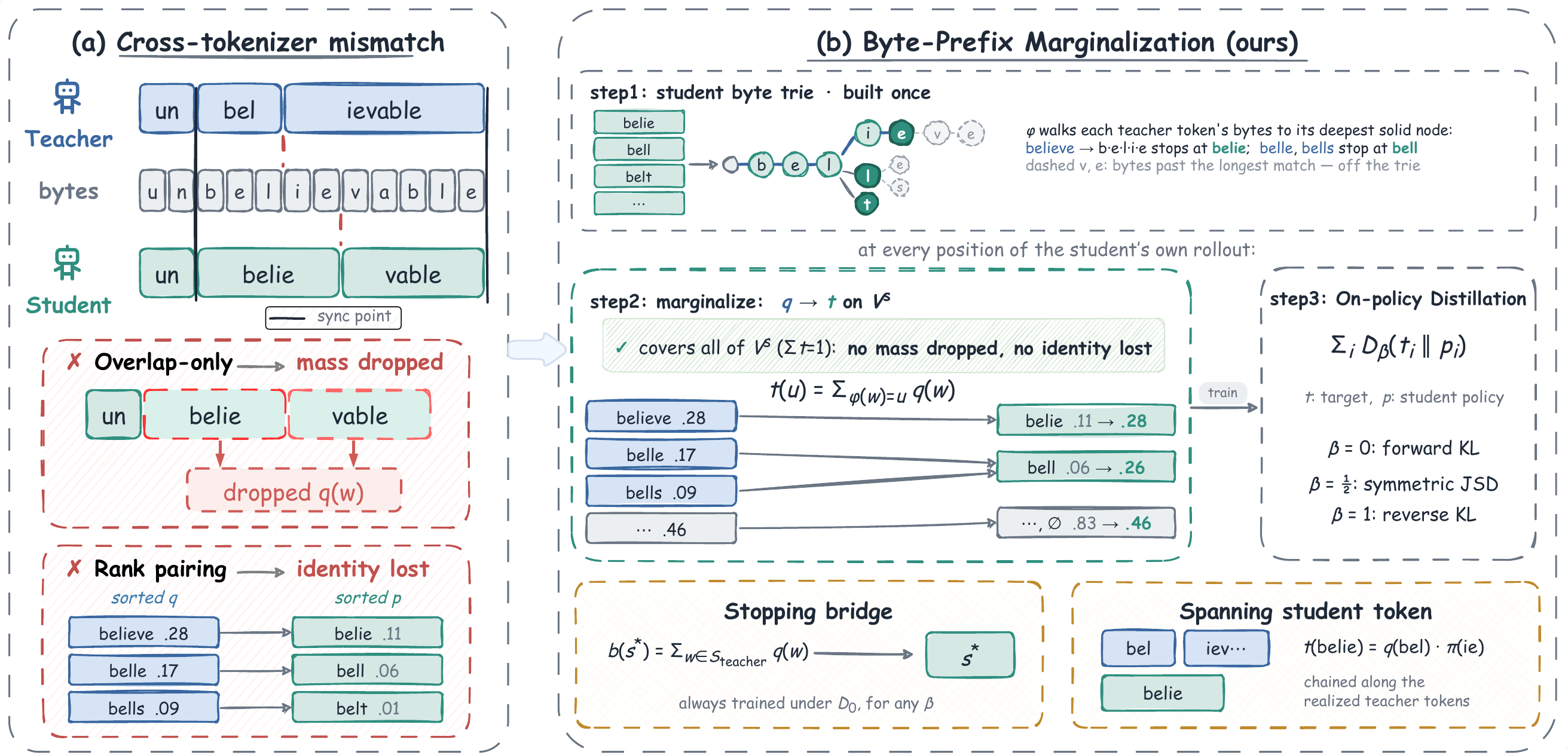}
\caption{\method{} overview. \textbf{(a)}~Both models cut the same bytes into different tokens (real segmentations of our most dissimilar pair); boundaries both models share are \emph{sync points} (black lines). \textbf{(b)}~The map $\phi$ routes each teacher token $w$ to the longest student token whose bytes prefix $\bytes(w)$. Probabilities routed to the same student token are summed into $t$ (here $.17{+}.09$ for \emph{bell}); unmatched mass goes to $\varnothing$, and $p$ is trained toward $t$ under $D_\beta$.}
\label{fig:overview}
\end{figure*}

\subsection{Cross-Tokenizer Setup}
\label{sec:setup-notation}

\looseness=-1 On-policy distillation queries the teacher along responses generated by the student. Let $\Vt$ and $\Vs$ denote the teacher and student vocabularies. At teacher position $j$, the teacher predicts $q_j \in \Delta(\Vt)$; at student position $i$, the student predicts $p_i \in \Delta(\Vs)$. When both models share a tokenizer, their positions and vocabulary coordinates coincide, giving the standard per-position objective:
\begin{equation}
\mathcal{L}_{\text{OPD}} \;=\; \textstyle\sum_i D\big(q_i \,\|\, p_i\big).
\label{eq:opd}
\end{equation}
Two conditions make this comparison meaningful: $q_i$ and $p_i$ live on the same vocabulary, and index $i$ denotes the same text prefix for both models.

\looseness=-1 Different tokenizers break both conditions. Their distributions occupy different vocabularies (\emph{vocabulary mismatch}), while equal position indices may refer to different text prefixes because token boundaries occur at different byte offsets (\emph{segmentation mismatch}). We therefore view each response as one byte string $\mathbf{x}$ with teacher segmentation $w^\rmt_1,\dots,w^\rmt_M$ and student segmentation $w^\rms_1,\dots,w^\rms_N$. Both decode to $\mathbf{x}$, but their boundaries and token counts differ. A boundary shared by both segmentations is a \emph{sync point}, and the bytes between consecutive sync points form a \emph{chunk}. At a sync point, student position $i$ aligns with the teacher position $j(i)$ beginning at the same byte offset. Throughout, assumption (A2) models student segmentation by greedy longest-prefix matching: at each offset, the student emits the longest token in $\Vs$ whose bytes prefix the remaining stream. Under this rule, a teacher prediction at an aligned position induces a distribution over the student's next token. The remaining task is to compute that distribution explicitly.

\subsection{Byte-Prefix Marginalization}
\label{sec:bpm-target}

\looseness=-1 Every byte-level BPE token deterministically decodes to a byte string. We write $\bytes(w)$ for the bytes of token $w$, $\mathbf{c}\preceq\mathbf{c}'$ when $\mathbf{c}$ is a prefix of $\mathbf{c}'$, and $\oplus$ for byte concatenation. Although the two tokenizers produce different token sequences, they decode each response to the same byte stream. \method{} uses this shared representation to marginalize teacher probability over the student token that greedy segmentation would emit first.

\vspace{-0.9\baselineskip}\paragraph{Definition.}
Fix an aligned student position $i$ and its corresponding teacher position $j(i)$. At this position, $q_{j(i)}(w)$ assigns probability to every possible teacher token $w$. Each candidate $w$ is mapped to the first student token obtained by greedily segmenting $\bytes(w)$. Candidates that do not contain enough bytes to determine a complete student token are assigned to an explicit residual category. The resulting target is defined as:
\begin{equation}
t_i(u) \;:=\; \Pr_{w \sim q_{j(i)}}\!\big[\text{first student token of } \bytes(w) = u\big].
\label{eq:target-def}
\end{equation}

\looseness=-1 Thus, $t_i(u)$ collects the probability of every teacher-token outcome that makes $u$ appear first under the student's greedy segmentation. This \emph{byte-prefix marginal} re-expresses the teacher's next-token distribution on the student vocabulary without matching token ids or probability ranks.

\vspace{-0.7\baselineskip}\paragraph{Closed form.}
\looseness=-1 The marginal has a one-step closed form in the \emph{refinement regime}, where each student token lies within a single teacher token and the student therefore segments at least as finely as the teacher. Whenever a candidate teacher token $w$ contains enough bytes to determine a complete student token, greedy matching selects the longest $u\in\Vs$ whose bytes prefix $\bytes(w)$. This defines:
\begin{equation}
\phi(w) \;=\; \operatorname*{arg\,max}_{u \in \Vs,\; \bytes(u) \preceq \bytes(w)} \lvert \bytes(u) \rvert,
\label{eq:phi}
\end{equation}

\looseness=-1 with $\phi(w)=\bot$ when no student token qualifies. The map $\phi$ can be compiled once for a tokenizer pair because it depends only on the two vocabularies. Grouping teacher outcomes by their mapped student token reduces \Cref{eq:target-def} to:
\begin{equation}
t_i(u) = \!\!\sum_{w:\,\phi(w)=u}\!\! q_{j(i)}(w),
\qquad
t_i(\varnothing) = \!\!\sum_{w:\,\phi(w)=\bot}\!\! q_{j(i)}(w).
\label{eq:target}
\end{equation}

\looseness=-1 The first sum pools all teacher tokens routed to $u$, while $\varnothing$ receives outcomes for which no complete student-token prefix exists. Every teacher token is assigned either to exactly one $u$ or to $\varnothing$, so the transformed mass still sums to one. Stop-token mass is handled separately in \Cref{sec:objective}.

\vspace{-0.7\baselineskip}\paragraph{Positions inside a teacher token.}
\looseness=-1 Refinement can place a later student position inside the same teacher token. Let $\mathbf{r}_{\mathrm{pre}}$ be the bytes already emitted since the preceding teacher boundary, and let $\pi(\mathbf{c})$ be the teacher mass on tokens beginning with $\mathbf{c}$. For each compatible token $w$, let $\phi_{\mathbf{r}_{\mathrm{pre}}}(w)$ be its longest student-token continuation, or $\bot$ if none exists. The conditional target is:
\begin{equation}
t_i(u \mid \mathbf{r}_{\mathrm{pre}})
 = \frac{\sum_{w:\,\mathbf{r}_{\mathrm{pre}}\preceq\bytes(w),\;
\phi_{\mathbf{r}_{\mathrm{pre}}}(w)=u} q(w)}{\pi(\mathbf{r}_{\mathrm{pre}})},
\qquad
t_i(\varnothing \mid \mathbf{r}_{\mathrm{pre}})
 = 1-\!\sum_{u\in\Vs} t_i(u \mid \mathbf{r}_{\mathrm{pre}}).
\label{eq:cond}
\end{equation}
\looseness=-1 Each compatible teacher token is counted once, so \Cref{eq:cond} is normalized; \Cref{app:notation} gives the equivalent two-cell scatter. Together, the aligned-boundary and within-token cases cover more than $99\%$ of training positions. \Cref{sec:mismatch} handles the remaining spanning cases.

\subsection{Spanning-Token Correction}
\label{sec:mismatch}

Refinement fails when one student token spans several teacher tokens. For example, the student tokenizer represents \texttt{</think>} as one token, whereas the teacher tokenizer splits the same bytes into three. The scatter in \Cref{eq:target} examines only teacher tokens at the aligned position, where almost no single token begins with the full bytes of \texttt{</think>}. It therefore assigns this student token near-zero mass, incorrectly teaching the student not to close its reasoning span.

The required spanning mass follows from the chain rule along the teacher's realized segmentation. Let student token $u$ begin at teacher position $j$, fully cover $w^\rmt_j,\dots,w^\rmt_{j+m-1}$, and end after a residual byte prefix $\mathbf{r}$ of the next teacher token. Write $\pi_k$ for the prefix mass from \Cref{eq:cond} evaluated at teacher position $k$. The chain value is defined as:
\begin{equation}
t_i(u) \;=\; \Big(\prod_{k=j}^{j+m-1} q_k\big(w^\rmt_k\big)\Big)\cdot \pi_{j+m}(\mathbf{r}),
\label{eq:chain}
\end{equation}
which is the probability that the teacher produces $\bytes(u)$ through this realized token path.

Other teacher segmentations may produce the same bytes, so \Cref{eq:chain} is a lower bound on the full byte marginal. We add this value to the target mass of $u$ and subtract the same amount from the closed-form target in \Cref{eq:target}, preserving total mass. Positions that neither begin at nor reach a sync point admit neither target and are excluded from the loss. Spanning corrections and excluded positions together account for less than one percent of response positions.

\subsection{Training Objective}
\label{sec:objective}

Training matches the transformed target with the student distribution through the generalized Jensen--Shannon divergence. For $0<\beta<1$, it is defined as:
\begin{equation}
D_\beta(t \,\|\, p) = \beta\, \mathrm{KL}(t \,\|\, m_\beta) + (1{-}\beta)\, \mathrm{KL}(p \,\|\, m_\beta),
\label{eq:jsd}
\end{equation}
with mixture $m_\beta = \beta t + (1{-}\beta) p$. We set $D_0(t\|p)=\mathrm{KL}(t\|p)$ (FKL) and use the skew reverse-KL endpoint $D_1(t\|p)=\mathrm{KL}\bigl(p\,\|\,(1{-}\lambda)t+\lambda p\bigr)$ with $\lambda=0.1$; \Cref{app:notation} gives their normalized-limit interpretation. The default is $\beta=\tfrac12$.

Each divergence uses the student's projection $\widetilde p_i$ onto the target's explicit token cells plus one residual complement, as defined in \Cref{app:notation}.
    
When the realized student token contains only spaces or tabs, \method{} withholds the target. This mask is precomputed from the student tokenizer and applied independently of the divergence. It covers about $3.5\%$ of aligned rows on the code half and is rare elsewhere, so in practice it acts almost entirely on code (\Cref{sec:collapse}).

Stop tokens require separate treatment because they carry no content bytes and therefore fall outside byte alignment. Moreover, the student's realized stop token need not equal the tokenizer's declared \texttt{eos}. We collect teacher and student stop sets $\mathcal{S}^\rmt, \mathcal{S}^\rms$ from the declared \texttt{eos}, generation config, and chat template, then remove those tokens from the content byte stream. Extending $\phi$ with $\phi(w)=s^*$ for every $w\in\mathcal{S}^\rmt$ maps all teacher stop tokens to the student's realized stop token $s^*$ for that response. The teacher's stopping probability can then join the target in \Cref{eq:target} at every aligned position and train under the same $D_\beta$ as content.

If stopping has already collapsed, an explicit bridge target supplies a restoring signal. At the position $i$ that emits $s^*$, \method{} defines:
\begin{equation}
b_i(s^*) = \!\sum_{w \in \mathcal{S}^\rmt}\! q_{j(i)}(w),
\qquad
b_i(\varnothing) = 1-b_i(s^*).
\label{eq:stop}
\end{equation}
The corresponding student projection is $\bar p_i=[p_i(s^*),1-p_i(s^*)]$. The bridge always trains under $D_0$, regardless of the run's $\beta$, because only this endpoint retains a restoring gradient after stopping collapses (\Cref{app:algo}). A healthy pipeline does not depend on the bridge.

Finally, let $\mathcal{R}$ collect the unmasked content positions and $\mathcal{R}'$ the stop positions of every response in the batch. The normalizer $Z$ counts all student response tokens, including masked and excluded positions that appear in neither sum. The complete objective is defined as:
\begin{equation}
\mathcal{L} \;=\; \tfrac{1}{Z} \textstyle\sum_{i \in \mathcal{R}} D_\beta(t_i \,\|\, \widetilde p_i)
\;+\; \tfrac{1}{Z} \textstyle\sum_{i \in \mathcal{R}'} D_0(b_i \,\|\, \bar p_i).
\label{eq:loss}
\end{equation}

\section{Experiments}
\label{sec:setup}

\subsection{Experimental Setup}
\label{sec:exp-setup}

\paragraph{Teacher--student pairs.}
To isolate the effect of the objective, every on-policy arm within a teacher--student pair uses the same student, prompt set, schedule, and update budget. Across the three pairs, only the teacher changes. We use Qwen3.5-2B~\citep{qwen2026qwen35card} as the fixed student, running this $248$k-token-vocabulary model in thinking mode for every rollout and evaluation. Tokenizer distance increases across the teachers. Qwen3-32B~\citep{qwen2025qwen3} shares the student's family and has span Jaccard $0.92$ between segmentations. GLM-Z1-9B-0414~\citep{glm2024chatglm,zai2025glmz1} comes from a different family yet stays close in segmentation (span Jaccard $0.87$), whereas MiniMax-M2.7~\citep{minimax2026m2series} sits farthest ($0.81$). We abbreviate the second teacher as GLM-Z1-9B and denote the pairs P1, P2, and P3, respectively, in later figures and tables. Because teacher capability also varies, all comparisons are made within a pair.

\vspace{-0.7\baselineskip}\paragraph{Training data.}
Training uses $20{,}000$ prompts, split evenly between mathematics and code. We draw $10$k problems with exact integer answers from DAPO-Math~\citep{yu2025dapo} and $10$k TACO problems~\citep{li2023taco}, each with a recovered executable test suite.

\vspace{-0.7\baselineskip}\paragraph{Implementation details.}
Except for SeqKD, every arm trains on-policy. Rollouts are sampled from the student at temperature $1.0$ under its release decode settings, with a maximum length of $27{,}648$ tokens. After each response is re-rendered in the teacher's chat template, the teacher scores the same generated content. The on-policy arms train for two epochs at batch size $256$ with a constant learning rate of $5{\times}10^{-7}$ on H200 SXM nodes. SeqKD follows its offline recipe with no smaller budget (\Cref{app:impl}). Unless stated otherwise, \method{} uses $\beta=\tfrac12$ and the whitespace-row mask (ablated below).

\vspace{-0.7\baselineskip}\paragraph{Baselines.}
We compare against the three cross-tokenizer baselines from Related Work and a sequence-level reference. The on-policy baselines retain their published losses and coefficients under the shared schedule. Rank-based \textbf{ULD}~\citep{boizard2024uld} matches sorted probability vectors under a transport loss. Hybrid \textbf{GOLD}~\citep{patino2025gold} applies a generalized Jensen--Shannon divergence on surface-matched tokens and the sorted loss on the remainder. Overlap-based \textbf{SimCT}~\citep{sun2026simct} applies reverse KL on the span-aligned shared-surface vocabulary. As the separately trained sequence-level reference, \textbf{SeqKD}~\citep{kim2016sequence} fine-tunes the student on teacher-generated responses to the same prompts, without constructing a distributional target. We evaluate \method{} with the three instantiations of \Cref{eq:jsd} (FKL, JSD, RKL) while holding the target fixed.

\vspace{-0.7\baselineskip}\paragraph{Evaluation protocol.}
Evaluation spans three mathematics and three code benchmarks. A symbolic verifier scores mathematics by accepting any equivalent expression, whereas code is evaluated by sandboxed execution against each benchmark's official tests and limits. The mathematics suite comprises MATH-500~\citep{hendrycks2021math,lightman2024verify} and the February 2026 AIME and HMMT contests~\citep{matharena2025}. For code, we use HumanEval+~\citep{liu2023evalplus}, the most recent official LiveCodeBench window~\citep{jain2025livecodebench}, and the official TACO test split~\citep{li2023taco} at its \textsc{easy}/\textsc{medium} labels. From $8$ samples per problem, we report \textbf{avg@8} and \textbf{pass@8}; each arm is represented by the checkpoint with the highest six-benchmark average. Decoding settings, metric definitions, and the paired-bootstrap protocol appear in \Cref{app:impl,app:ci}.

\begin{table*}[!t]
\centering
\small
\setlength{\tabcolsep}{1.6pt}
\renewcommand{\arraystretch}{1.0}
\begin{adjustbox}{max width=\textwidth}
\begin{tabular}{@{}l@{\,}l*{14}{c}@{}}
\toprule
 & & \multicolumn{6}{c}{Mathematics} & \multicolumn{6}{c}{Code} & \multicolumn{2}{c}{} \\
\cmidrule(lr){3-8} \cmidrule(lr){9-14}
 \multicolumn{2}{@{}l}{Method} & \multicolumn{2}{c}{AIME 2026} & \multicolumn{2}{c}{HMMT 2026} & \multicolumn{2}{c}{MATH-500} & \multicolumn{2}{c}{HumanEval+} & \multicolumn{2}{c}{LiveCodeBench} & \multicolumn{2}{c}{TACO} & \multicolumn{2}{c}{Avg.} \\
\cmidrule(lr){3-4} \cmidrule(lr){5-6} \cmidrule(lr){7-8} \cmidrule(lr){9-10} \cmidrule(lr){11-12} \cmidrule(lr){13-14} \cmidrule(lr){15-16}
 &  & avg@8 & pass@8 & avg@8 & pass@8 & avg@8 & pass@8 & avg@8 & pass@8 & avg@8 & pass@8 & avg@8 & pass@8 & avg@8 & pass@8 \\
\midrule
\multicolumn{2}{@{}l}{Qwen3.5-2B (base)} & 7.5 & 26.7 & 10.6 & 18.2 & 60.5 & 84.6 & 52.5 & 84.0 & 11.6 & 17.6 & 5.3 & 11.7 & 24.7 & 40.5 \\
\midrule
\rowcolor{teacherbg}\multicolumn{2}{@{}l}{Qwen3-32B} & 71.7 & 86.7 & 47.3 & 63.6 & 93.7 & 97.2 & 82.2 & 96.9 & 56.5 & 73.1 & 53.2 & 64.3 & 67.4 & 80.3 \\
\multicolumn{2}{@{}l}{\quad SimCT} & \snd{35.8} & \snd{56.7} & \snd{18.9} & \snd{30.3} & \best{83.4} & 93.6 & 56.4 & 87.7 & 10.6 & 24.7 & 6.7 & 21.2 & \snd{35.3} & \snd{52.4} \\
\multicolumn{2}{@{}l}{\quad ULD} & 7.5 & 20.0 & 6.1 & 15.2 & 60.7 & 89.0 & 55.9 & 79.8 & 12.0 & 25.3 & 8.2 & 19.1 & 25.1 & 41.4 \\
\multicolumn{2}{@{}l}{\quad GOLD} & 12.5 & 33.3 & 6.8 & 21.2 & 71.0 & 91.2 & \best{62.1} & \best{89.6} & 14.1 & 23.6 & 11.7 & 24.7 & 29.7 & 47.3 \\
\multicolumn{2}{@{}l}{\quad SeqKD} & 22.5 & 43.3 & 17.0 & 30.3 & 77.8 & \snd{94.2} & 50.8 & 81.0 & \best{18.8} & \snd{29.7} & \best{15.7} & \best{35.7} & 33.8 & 52.4 \\
\multicolumn{2}{@{}l}{\quad \method{}} & \best{37.5} & \best{73.3} & \best{23.5} & \best{39.4} & \snd{82.1} & \best{96.4} & \snd{61.8} & \snd{88.3} & \snd{15.7} & \best{33.0} & \snd{13.3} & \snd{35.0} & \best{39.0} & \best{60.9} \\
\multicolumn{2}{@{}l}{\quad\textcolor{gray}{$\Delta$ vs.\ best baseline}} & \textcolor{gray}{$+1.7$} & \textcolor{gray}{$+16.6$} & \textcolor{gray}{$+4.6$} & \textcolor{gray}{$+9.1$} & \textcolor{gray}{$-1.3$} & \textcolor{gray}{$+2.2$} & \textcolor{gray}{$-0.3$} & \textcolor{gray}{$-1.3$} & \textcolor{gray}{$-3.1$} & \textcolor{gray}{$+3.3$} & \textcolor{gray}{$-2.4$} & \textcolor{gray}{$-0.7$} & \textcolor{gray}{$+3.7$} & \textcolor{gray}{$+8.5$} \\
\midrule
\rowcolor{teacherbg}\multicolumn{2}{@{}l}{GLM-Z1-9B} & 63.3 & 90.0 & 33.3 & 48.5 & 92.3 & 97.4 & 89.8 & 96.9 & 45.6 & 63.2 & 53.5 & 64.0 & 63.0 & 76.7 \\
\multicolumn{2}{@{}l}{\quad SimCT} & 15.8 & 30.0 & 12.5 & 27.3 & 58.6 & 89.2 & 50.2 & 78.5 & 12.6 & 23.6 & 8.4 & 21.2 & 26.3 & 45.0 \\
\multicolumn{2}{@{}l}{\quad ULD} & 17.5 & 40.0 & 12.5 & 24.2 & 81.2 & 94.4 & \snd{66.0} & 88.3 & 15.8 & 25.8 & 12.4 & 27.2 & 34.2 & 50.0 \\
\multicolumn{2}{@{}l}{\quad GOLD} & 22.1 & \snd{53.3} & \snd{17.4} & \snd{33.3} & \snd{81.8} & \best{96.4} & 63.0 & \snd{89.0} & 9.0 & 23.6 & 4.8 & 15.5 & 33.0 & 51.9 \\
\multicolumn{2}{@{}l}{\quad SeqKD} & \snd{23.3} & 50.0 & 14.8 & 27.3 & 74.4 & 93.4 & 60.1 & 87.1 & \snd{20.1} & \snd{31.3} & \best{16.8} & \snd{34.3} & \snd{34.9} & \snd{53.9} \\
\multicolumn{2}{@{}l}{\quad \method{}} & \best{35.4} & \best{63.3} & \best{21.2} & \best{36.4} & \best{84.8} & \snd{95.4} & \best{68.8} & \best{91.4} & \best{22.3} & \best{33.5} & \snd{16.5} & \best{36.4} & \best{41.5} & \best{59.4} \\
\multicolumn{2}{@{}l}{\quad\textcolor{gray}{$\Delta$ vs.\ best baseline}} & \textcolor{gray}{$+12.1$} & \textcolor{gray}{$+10.0$} & \textcolor{gray}{$+3.8$} & \textcolor{gray}{$+3.1$} & \textcolor{gray}{$+3.0$} & \textcolor{gray}{$-1.0$} & \textcolor{gray}{$+2.8$} & \textcolor{gray}{$+2.4$} & \textcolor{gray}{$+2.2$} & \textcolor{gray}{$+2.2$} & \textcolor{gray}{$-0.3$} & \textcolor{gray}{$+2.1$} & \textcolor{gray}{$+6.6$} & \textcolor{gray}{$+5.5$} \\
\midrule
\rowcolor{teacherbg}\multicolumn{2}{@{}l}{MiniMax-M2.7} & 75.4 & 93.3 & 48.1 & 69.7 & 94.1 & 98.4 & 90.7 & 97.5 & 42.3 & 64.8 & 41.5 & 59.4 & 65.4 & 80.5 \\
\multicolumn{2}{@{}l}{\quad SimCT} & 3.3 & 10.0 & 6.1 & 15.2 & 61.2 & 82.2 & 61.2 & 85.9 & 11.3 & 18.1 & 7.0 & 14.8 & 25.0 & 37.7 \\
\multicolumn{2}{@{}l}{\quad ULD} & 15.0 & 36.7 & 10.6 & 21.2 & \snd{76.5} & \snd{91.2} & \snd{63.9} & 86.5 & \snd{17.3} & \snd{25.3} & \snd{13.6} & \snd{26.9} & \snd{32.8} & \snd{48.0} \\
\multicolumn{2}{@{}l}{\quad GOLD} & 9.2 & 30.0 & 6.1 & 18.2 & 69.6 & 87.6 & 63.5 & \snd{89.0} & 11.5 & 18.7 & 5.2 & 11.3 & 27.5 & 42.5 \\
\multicolumn{2}{@{}l}{\quad SeqKD} & \snd{18.3} & \snd{40.0} & \snd{11.0} & \snd{24.2} & 64.5 & 90.6 & 46.4 & 81.0 & 12.4 & 23.1 & 6.0 & 14.1 & 26.4 & 45.5 \\
\multicolumn{2}{@{}l}{\quad \method{}} & \best{25.4} & \best{46.7} & \best{18.6} & \best{39.4} & \best{81.8} & \best{95.6} & \best{67.1} & \best{92.0} & \best{21.8} & \best{34.1} & \best{17.4} & \best{34.6} & \best{38.7} & \best{57.1} \\
\multicolumn{2}{@{}l}{\quad\textcolor{gray}{$\Delta$ vs.\ best baseline}} & \textcolor{gray}{$+7.1$} & \textcolor{gray}{$+6.7$} & \textcolor{gray}{$+7.6$} & \textcolor{gray}{$+15.2$} & \textcolor{gray}{$+5.3$} & \textcolor{gray}{$+4.4$} & \textcolor{gray}{$+3.2$} & \textcolor{gray}{$+3.0$} & \textcolor{gray}{$+4.5$} & \textcolor{gray}{$+8.8$} & \textcolor{gray}{$+3.8$} & \textcolor{gray}{$+7.7$} & \textcolor{gray}{$+5.9$} & \textcolor{gray}{$+9.1$} \\
\bottomrule
\end{tabular}
\end{adjustbox}
\caption{Main results: avg@8 and pass@8 (\%). Tinted rows report teacher performance; the base row is the shared undistilled student. For \method{}, we report forward KL (FKL), which performs best among its three divergence variants on all three pairs (\Cref{tab:div}). Within each teacher block, dark shading and boldface mark the best distilled result in each metric column, while light shading and underlining mark the second best. Gray $\Delta$ rows report \method{} minus the strongest baseline in each column.}
\label{tab:main}
\end{table*}

\begin{table}[t]
\begin{minipage}[t]{0.545\textwidth}
\centering
\footnotesize
\renewcommand{\arraystretch}{0.95}
\setlength{\tabcolsep}{2.6pt}
\begin{adjustbox}{max width=\linewidth, valign=t}
\begin{tabular}{@{}l cc cc cc@{}}
\toprule
 & \multicolumn{2}{c}{Math (3)} & \multicolumn{2}{c}{Code (3)} & \multicolumn{2}{c}{All (6)} \\
\cmidrule(lr){2-3} \cmidrule(lr){4-5} \cmidrule(lr){6-7}
 & avg@8 & pass@8 & avg@8 & pass@8 & avg@8 & pass@8 \\
\midrule
\rowcolor{teacherbg}Qwen3-32B & 70.9 & 82.5 & 64.0 & 78.1 & 67.4 & 80.3 \\
\quad FKL & \best{47.7} & \best{69.7} & \best{30.3} & \best{52.1} & \best{39.0} & \best{60.9} \\
\quad JSD & 46.1 & 65.9 & 25.5 & \snd{49.7} & 35.8 & 57.8 \\
\quad RKL & \snd{47.0} & \snd{67.0} & \snd{26.6} & 49.6 & \snd{36.8} & \snd{58.3} \\
\midrule
\rowcolor{teacherbg}GLM-Z1-9B & 63.0 & 78.6 & 63.0 & 74.7 & 63.0 & 76.7 \\
\quad FKL & \best{47.1} & \best{65.0} & \best{35.9} & \best{53.8} & \best{41.5} & \best{59.4} \\
\quad JSD & \snd{43.1} & \snd{60.7} & \snd{34.2} & \snd{50.8} & \snd{38.6} & \snd{55.8} \\
\quad RKL & 42.5 & 59.4 & 31.1 & 49.9 & 36.8 & 54.6 \\
\midrule
\rowcolor{teacherbg}MiniMax-M2.7 & 72.5 & 87.1 & 58.2 & 73.9 & 65.4 & 80.5 \\
\quad FKL & \snd{41.9} & \best{60.6} & \best{35.4} & \best{53.6} & \best{38.7} & \best{57.1} \\
\quad JSD & \best{43.7} & 59.1 & \snd{23.8} & \snd{48.8} & \snd{33.8} & \snd{54.0} \\
\quad RKL & 38.5 & \snd{59.8} & 18.6 & 41.7 & 28.5 & 50.7 \\
\bottomrule
\end{tabular}
\end{adjustbox}
\captionof{table}{Divergence ablation for \method{}: forward KL (FKL, $\beta{=}0$), JSD ($\beta{=}\tfrac12$), and reverse KL (RKL, $\beta{=}1$, skew $0.1$). Math and Code average the corresponding three benchmarks in \Cref{tab:main}; All averages all six. Tinted rows report teacher performance, and highlighting follows \Cref{tab:main}. The P3 RKL result is measured at its pre-collapse peak (\Cref{app:instability}).}
\label{tab:div}
\end{minipage}\hfill
\begin{minipage}[t]{0.415\textwidth}
\centering
\footnotesize
\renewcommand{\arraystretch}{0.95}
\setlength{\tabcolsep}{3.2pt}
\begin{adjustbox}{max width=\linewidth, valign=t}
\begin{tabular}{@{}l cccc@{}}
\toprule
 & FKL & JSD & RKL & Avg. \\
\midrule
\multicolumn{5}{@{}l}{\emph{Original mix ($20$k)}} \\
\quad avg@8 & \best{41.5} & 38.6 & 36.8 & \snd{39.0} \\
\quad pass@8 & \best{59.4} & 55.8 & \snd{54.6} & \snd{56.6} \\
\midrule
\multicolumn{5}{@{}l}{\emph{Harder-math mix ($20$k)}} \\
\quad avg@8 & 39.4 & \best{38.9} & \snd{37.4} & 38.6 \\
\quad\textcolor{gray}{$\Delta$} & \textcolor{gray}{$-2.1$} & \textcolor{gray}{$+0.3$} & \textcolor{gray}{$+0.6$} & \textcolor{gray}{$-0.4$} \\
\quad pass@8 & 57.1 & \snd{57.0} & \best{55.1} & 56.4 \\
\quad\textcolor{gray}{$\Delta$} & \textcolor{gray}{$-2.3$} & \textcolor{gray}{$+1.2$} & \textcolor{gray}{$+0.5$} & \textcolor{gray}{$-0.2$} \\
\midrule
\multicolumn{5}{@{}l}{\emph{Expanded corpus ($35.5$k)}} \\
\quad avg@8 & \snd{40.7} & \best{38.9} & \best{37.6} & \best{39.1} \\
\quad\textcolor{gray}{$\Delta$} & \textcolor{gray}{$-0.8$} & \textcolor{gray}{$+0.3$} & \textcolor{gray}{$+0.8$} & \textcolor{gray}{$+0.1$} \\
\quad pass@8 & \snd{58.4} & \best{58.3} & 53.9 & \best{56.9} \\
\quad\textcolor{gray}{$\Delta$} & \textcolor{gray}{$-1.0$} & \textcolor{gray}{$+2.5$} & \textcolor{gray}{$-0.7$} & \textcolor{gray}{$+0.3$} \\
\bottomrule
\end{tabular}
\end{adjustbox}
\captionof{table}{Data ablations on the GLM-Z1-9B pair, reported as six-benchmark avg@8 and pass@8. FKL, JSD, and RKL are the three \method{} divergence variants; Avg. averages their columns. Gray $\Delta$ rows give the change from the original mix, and highlighting follows \mbox{\Cref{tab:main}}.}
\label{tab:datamix}
\end{minipage}
\end{table}

\subsection{Distillation Performance}
\label{sec:results}

Under the matched budget shared by the on-policy arms, the best \method{} variant beats all three distributional baselines on every pair and also exceeds the separately trained SeqKD reference (\Cref{tab:main,tab:div}). Across P1--P3, it closes $33.5\%$, $43.9\%$, and $34.4\%$ of the teacher--student gap, compared with $19.9$--$26.6\%$ for the strongest baselines (\Cref{fig:collapse}d). Improvements are largest on the contest benchmarks.

Even the default $\beta{=}\tfrac12$ arm outperforms the strongest baseline on every pair, and it stays ahead on the most dissimilar pair ($33.8$ against ULD's $32.8$). With the best divergence selected, the six-benchmark margins reach $+3.7$, $+6.6$, and $+5.9$ points (\Cref{tab:main}). All three reported best-arm average intervals exclude zero; \Cref{app:ci} specifies their comparators.

Two data ablations on the GLM-Z1-9B pair test whether the divergence ranking depends on the original prompt set (\Cref{tab:datamix}). Replacing the mathematics half with harder competition-style problems shifts the six-benchmark average by at most two points ($-2.1$ under FKL, $+0.3$ under JSD, $+0.6$ under RKL), while expanding the corpus from $20$k to $35.5$k prompts ($1.78\times$) shifts it by at most one ($-0.8$, $+0.3$, $+0.8$). Averaged across the three divergences, the respective changes are only $-0.4$ and $+0.1$. Both perturbations lower FKL and raise JSD and RKL, narrowing the gaps without changing the FKL $>$ JSD $>$ RKL ordering. At this budget, divergence choice therefore matters more than either data change.

\begin{wraptable}{r}{0.5\textwidth}
\vspace{-0.9\baselineskip}
\centering
\footnotesize
\setlength{\tabcolsep}{3.2pt}
\begin{adjustbox}{max width=0.495\textwidth}
\begin{tabular}{@{}lcccc@{}}
\toprule
Method & Tea.\ BpB\,$\downarrow$ & Shared\,$\uparrow$ & Byte agr.\,$\uparrow$ & Stu.\ BpB\,$\downarrow$ \\
\midrule
Base & $0.449$ & $0.930$ & $0.933$ & $0.242$ \\
SimCT & \best{$0.155$} & \snd{$0.952$} & \best{$0.977$} & $0.248$ \\
ULD & $0.231$ & $0.933$ & $0.963$ & $0.241$ \\
GOLD & \snd{$0.226$} & \best{$0.954$} & \snd{$0.969$} & \snd{$0.238$} \\
\method{} & $0.250$ & $0.943$ & $0.964$ & \best{$0.237$} \\
\bottomrule
\end{tabular}
\end{adjustbox}
\caption{Distributional closeness on the GLM-Z1-9B pair. Teacher BpB, shared teacher mass, and byte-level top-1 agreement are measured on student rollouts; student BpB is measured on teacher-written text. Lower BpB and higher shared mass or agreement are the nominal closer-to-teacher directions.}
\label{tab:closeness}
\vspace{-0.5\baselineskip}
\end{wraptable}%
\looseness=-1 Beyond accuracy, \Cref{tab:closeness} scores four distributional-closeness quantities on the GLM-Z1-9B pair. Three of them reward any student the teacher finds predictable, and a student collapsed into repetition is easy to predict. The most collapsed baseline (SimCT) therefore posts the best-looking teacher bits-per-byte ($0.155$). A collapsed student cannot game the fourth quantity, its own likelihood of teacher-written text. That quantity inverts the order. \method{} ranks first among the distilled arms ($0.237$), while SimCT falls below the base ($0.248$ against $0.242$). \Cref{app:closeness} extends the comparison to the other two pairs.

\looseness=-1 The marginalized stopping term of \Cref{sec:objective} folds the teacher's end-of-turn decision into the target, and the students inherit that decision. On the Qwen3-32B pair, \method{} cuts the fraction of AIME 2026 responses that hit the length cap from $97\%$ to $34\%$, against $49$--$62\%$ for the baselines. Response lengths otherwise stay unchanged. GLM-Z1-9B never caps, and its \method{} student reaches $8\%$. MiniMax-M2.7 runs $55\%$ of its own HMMT responses to the cap, and its students stay above $60\%$.

\subsection{Whitespace Collapse}
\label{sec:collapse}

Among code examples in the mixed on-policy rollouts, executable correctness rises through the first epoch before falling in the second. For the GLM-Z1-9B pair, the decline begins at step $113$ of $156$ and settles below $0.06$ (\Cref{fig:collapse}). Meanwhile, mathematics accuracy on the same rollouts remains stable, and the generated code remains syntactically valid. Closer inspection traces this selective loss of executable correctness to indentation.

The collapse concentrates at all-whitespace rows, where byte-exact supervision no longer conveys a lexical content preference. At a content row, target mass away from the student's realized token indicates what the teacher would rather write. At an all-whitespace row, by contrast, this mass reflects tokenizer-specific choices among whitespace continuations, conflating token segmentation with indentation depth. During the collapse on the GLM-Z1-9B pair, the target retains $73\%$ of its mass on the realized token on average, but this share falls below half on a quarter of the rows. The next-token entropy at these rows is only $0.53$.     

\begin{figure*}[!t]
\centering
\includegraphics[width=0.85\textwidth]{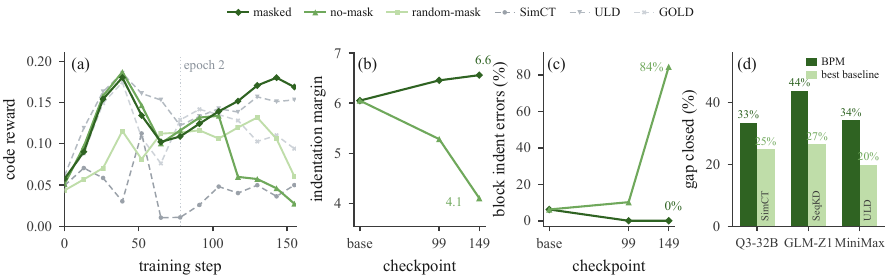}
\caption{(a--c)~Whitespace collapse for GLM-Z1-9B \method{}-JSD with/without whitespace-row masking: (a)~rollout code reward ($9$-step average; dotted: epoch 2), (b)~mean indentation margin over $1{,}799$ positions, and (c)~block indentation errors. (d)~Teacher--student avg@8 gap closure for each pair's best \method{} arm and strongest baseline (named in bars).}
\label{fig:collapse}
\end{figure*}

Training toward this target erodes the student's indentation margins, and line-level errors accumulate across each program. Block-level indentation errors rise to $84\%$, while the pass rate falls from $49\%$ to $9\%$ (\Cref{fig:collapse}b,c). Masking all-whitespace rows prevents the collapse: although the mask removes only about $3.5\%$ of aligned code rows, the otherwise identical masked arm finishes with an indentation margin of $6.57$ and no block-level indentation errors.

Three controls isolate the effect. The masked and unmasked arms differ only in the whitespace-row mask. Masking the same number of random non-whitespace rows per response, matched for loss reduction, does not rescue the arm; executable correctness still declines late and ends at a $0.11$ tail mean against $0.17$. SimCT, ULD, and GOLD show no comparable collapse under the same mix, and none receives the byte-exact target.

\section{Conclusion}
\label{sec:conclusion}
\looseness=-1 We distill on-policy across a tokenizer boundary. \method{} marginalizes the teacher's next-token distribution over byte prefixes and re-expresses it on the student's vocabulary. Each teacher token's mass is routed to its longest student-token byte prefix or to the residual, conserving total mass. Our experiments distill Qwen3-32B, GLM-Z1-9B, and MiniMax-M2.7 into one thinking-mode student. The three teachers sit at increasing tokenizer distance from the student. Under the budget shared by the on-policy arms, \method{} beats the overlap, rank, and hybrid baselines on every pair and also exceeds the separately trained SeqKD reference. The margins over each pair's strongest baseline run from $+3.7$ to $+6.6$ avg@8 points and are widest on contests held after every teacher's training cutoff. The target also carries the teacher's decision to stop, and the students inherit it. Fidelity backfires where the bytes carry no content. At whitespace-only code positions the target encodes only segmentation noise, and training on that noise collapses executable code. A precomputed whitespace-row mask removes the collapse. Both the routing map and the mask come from the two tokenizers alone, before training. Other cross-tokenizer objectives will face the same choice of where to apply fidelity and where to withhold it.

\vspace{-0.45\baselineskip}\paragraph{Limitations.}
\looseness=-1 First, \method{} discards part of the gradient. The whitespace-row mask withholds the target on about $3.5\%$ of aligned code rows. Positions that reach no sync point leave the loss entirely. Whatever signal these rows carry never reaches the student. Second, data quality bounds what our experiments can show. Whether the remaining teacher--student gap survives higher-quality data and tasks is untested. Such data can sharpen the student's next-token prediction, and sharper prediction amplifies on-policy distillation.

\bibliographystyle{unsrtnat}
\bibliography{references}

\clearpage
\appendix
\counterwithin{equation}{section}
\counterwithin{table}{section}
\counterwithin{figure}{section}

The appendix follows the dependency chain behind the main result. We first restate the notation and target definitions (\Cref{app:notation}), turn them into a complete training algorithm (\Cref{app:algo}), and then establish the exactness and mass-conservation claims (\Cref{app:proof}). The segmentation dispatch is made concrete through worked examples (\Cref{app:cases}), while a fixed teacher distribution exposes the differences among target constructions (\Cref{app:targets}). Finally, we document implementation cost, prompts, and evaluation details (\Cref{app:impl,app:prompts}) before expanding the empirical results (\Cref{app:extra-results}).

\section{Notation Recap and Restated Definitions}
\label{app:notation}

To make the later derivations self-contained, we begin with the shared byte representation. The teacher $T$ has vocabulary $\Vt$, the student $S$ has vocabulary $\Vs$, and every token $w$ from either byte-level BPE tokenizer decodes losslessly to $\bytes(w)$. We write $\mathbf{c} \preceq \mathbf{c}'$ when $\mathbf{c}$ is a byte prefix of $\mathbf{c}'$, and $\oplus$ for byte concatenation. At teacher position $j$, the teacher predicts $q_j \in \Delta(\Vt)$; at student position $i$, the student predicts $p_i \in \Delta(\Vs)$. A byte offset shared by both token boundaries is a \emph{sync point}. Student position $i$ at such an offset aligns with teacher position $j(i)$, and consecutive sync points delimit a \emph{chunk}. The teacher's \emph{prefix mass} at position $k$ is defined as:
\begin{equation}
\pi_k(\mathbf{c}) \;=\; \sum_{w \in \Vt:\, \mathbf{c} \preceq \bytes(w)} q_k(w),
\label{eq:s-prefixmass}
\end{equation}
the probability that its next token starts with $\mathbf{c}$; we drop the subscript when the position is clear.

With this notation fixed, the refinement target can be restated in three steps. First, the longest-prefix map from \Cref{eq:phi} is defined as:
\begin{equation}
\phi(w) \;=\; \operatorname*{arg\,max}_{u \in \Vs,\; \bytes(u) \preceq \bytes(w)} \lvert \bytes(u) \rvert,
\label{eq:s-phi}
\end{equation}
with $\phi(w)=\bot$ when no student token qualifies. Second, the target at an aligned student position $i$ from \Cref{eq:target-def} is defined as:
\begin{equation}
t_i(u) \;:=\; \Pr_{w \sim q_{j(i)}}\!\big[\text{first student token of } \bytes(w) = u\big],
\label{eq:s-def}
\end{equation}
Marginalizing the teacher outcomes through $\phi$ gives the third step, the closed scatter form from \Cref{eq:target}:
\begin{equation}
t_i(u) = \!\!\sum_{w:\,\phi(w)=u}\!\! q_{j(i)}(w),
\qquad
t_i(\varnothing) = \!\!\sum_{w:\,\phi(w)=\bot}\!\! q_{j(i)}(w).
\label{eq:s-scatter}
\end{equation}
The scatter above covers tokens resolved within one teacher step. When a student token $u$ instead starts at teacher position $j$, covers realized teacher tokens $w^\rmt_j,\dots,w^\rmt_{j+m-1}$, and extends through a residual byte prefix $\mathbf{r}$ of the next token, it receives the chain value from \Cref{eq:chain}:
\begin{equation}
t_i(u) \;=\; \Big(\prod_{k=j}^{j+m-1} q_k\big(w^\rmt_k\big)\Big)\cdot \pi_{j+m}(\mathbf{r}),
\label{eq:s-chain}
\end{equation}
For a student position beginning \emph{inside} a teacher token after realized bytes $\mathbf{r}_{\mathrm{pre}}$, let $\phi_{\mathbf{r}_{\mathrm{pre}}}(w)$ return the longest student token that prefixes the remaining bytes of compatible teacher token $w$, or $\bot$ if none exists. The exact conditional scatter is defined as:
\begin{equation}
\begin{aligned}
t_i(u \mid \mathbf{r}_{\mathrm{pre}})
&= \frac{\sum_{w:\,\mathbf{r}_{\mathrm{pre}}\preceq\bytes(w),\;
\phi_{\mathbf{r}_{\mathrm{pre}}}(w)=u} q(w)}{\pi(\mathbf{r}_{\mathrm{pre}})},\\
t_i(\varnothing \mid \mathbf{r}_{\mathrm{pre}})
&= \frac{\sum_{w:\,\mathbf{r}_{\mathrm{pre}}\preceq\bytes(w),\;
\phi_{\mathbf{r}_{\mathrm{pre}}}(w)=\bot} q(w)}{\pi(\mathbf{r}_{\mathrm{pre}})}.
\end{aligned}
\label{eq:s-cond}
\end{equation}
The two cells partition the conditioned teacher support and therefore sum to one without counting nested student-token prefixes more than once. Finally, let $\mathcal{S}^\rmt$ and $\mathcal{S}^\rms$ be the declared stop sets and $s^*$ the student's realized stop token for the response at hand. The stop bridge and its student-side binary projection from \Cref{eq:stop} are defined as:
\begin{equation}
\begin{aligned}
b_i(s^*) &= \!\sum_{w \in \mathcal{S}^\rmt}\! q_{j(i)}(w),
& b_i(\varnothing) &= 1-b_i(s^*),\\
\bar p_i(s^*) &= p_i(s^*),
& \bar p_i(\varnothing) &= 1-p_i(s^*).
\end{aligned}
\label{eq:s-stop}
\end{equation}

\paragraph{Divergence endpoints and projected support.}
The generalized Jensen--Shannon expression in \Cref{eq:jsd} is used for $0<\beta<1$. Because it vanishes at the endpoints, FKL is defined by the normalized limit $D_0(t\|p):=\lim_{\beta\to0}D_\beta(t\|p)/\beta=\mathrm{KL}(t\|p)$. The other normalized limit, $\lim_{\beta\to1}D_\beta(t\|p)/(1-\beta)=\mathrm{KL}(p\|t)$, motivates the numerically stable skew-RKL endpoint:
\begin{equation}
D_1(t\|p):=\mathrm{KL}\!\left(p\,\middle\|\,(1-\lambda)t+\lambda p\right),
\qquad \lambda=0.1.
\label{eq:s-endpoints}
\end{equation}
For a content row, let $\mathcal{I}_i$ contain the student tokens represented explicitly by its routing map and any spanning correction. The student distribution is projected onto these cells and one complement coordinate:
\begin{equation}
\widetilde p_i(u)=p_i(u)\quad(u\in\mathcal{I}_i),
\qquad
\widetilde p_i(\varnothing)=\!\sum_{v\in\Vs\setminus\mathcal{I}_i}\!p_i(v)
=1-\!\sum_{u\in\mathcal{I}_i}\!p_i(u).
\label{eq:s-projection}
\end{equation}
Thus $t_i$ and $\widetilde p_i$ have identical support and retain all teacher and student mass. The binary stop projection $\bar p_i$ is defined separately in \Cref{eq:s-stop}.

\section{Training Algorithm}
\label{app:algo}

\Cref{alg:bpm} gives the full training procedure: the per-pair precomputation, the per-response byte alignment, the per-position dispatch among the target constructions of \Cref{app:notation}, the whitespace-row mask, and the loss. The dispatch mirrors the case table of \Cref{app:cases} (\Cref{tab:cases}).

\begin{algorithm}[H]
\caption{\method{}: precomputation and one on-policy training step}
\label{alg:bpm}
\footnotesize
\begin{algorithmic}[1]
\Require teacher $q$ over $\Vt$, student $p_\theta$ over $\Vs$, prompt batch $\mathcal{B}$, mixing weight $\beta$
\Statex \textbf{Precompute once per tokenizer pair:}
\State build the byte trie of $\Vs$; for each $w \in \Vt$, descend $\bytes(w)$ to set $\phi(w)$ per \Cref{eq:s-phi}
\State read stop sets $\mathcal{S}^\rmt, \mathcal{S}^\rms$ from tokenizer, generation configuration, and chat template
\Statex \textbf{One training step:}
\State $\mathcal{L} \gets 0$
\For{each prompt $x \in \mathcal{B}$}
  \State sample response $y \sim p_\theta(\cdot \mid x)$
    \Comment{on-policy rollout}
  \State teacher forward on $y$ re-rendered in the teacher's chat template
    $\to q_1,\dots,q_M$; \ student forward $\to p_1,\dots,p_N$
  \State strip stop-token bytes from $y$; extend $\phi(w) \gets s^*$ for all
    $w \in \mathcal{S}^\rmt$, where $s^* \in \mathcal{S}^\rms$ ends $y$
  \State compute byte offsets of both segmentations; mark sync points and chunks
  \For{each student position $i = 1,\dots,N$}
    \If{$i$ is aligned and its token ends at or inside the aligned teacher token}
      \State $t_i \gets$ scatter of $q_{j(i)}$ through $\phi$
        \Comment{\Cref{eq:s-scatter}; $1{:}1$ and $1{:}N$ head}
    \ElsIf{$i$ starts inside a teacher token and its token ends at or inside that teacher token, realized prefix $\mathbf{r}_{\mathrm{pre}}$}
      \State $t_i \gets$ compatible teacher mass, normalized by $\pi(\mathbf{r}_{\mathrm{pre}})$, scattered through $\phi_{\mathbf{r}_{\mathrm{pre}}}$
        \Comment{\Cref{eq:s-cond}; $1{:}N$ interior}
    \ElsIf{the token at $i$ spans one or more teacher boundaries}
      \State $t_i \gets$ scatter of $q_{j(i)}$ through $\phi$
        \Comment{base row, \Cref{eq:s-scatter}}
      \State $t_i(u) \gets$ chain value of \Cref{eq:s-chain}, conditioned on
        $\mathbf{r}_{\mathrm{pre}}$ if the row starts mid-token;
      \State subtract the shallowest candidate's chain value from the head's
        scatter cell (or $\varnothing$ if unmapped), split across nested
        candidates by telescoping differences
        \Comment{mass-neutral; $N{:}1$ and $M{:}N$}
    \Else
      \State exclude $i$ from the loss
        \Comment{no sync point reachable; $<1\%$ of rows}
    \EndIf
    \If{$i$ emits $s^*$}
      \State $b_i(s^*) \gets \sum_{w \in \mathcal{S}^\rmt} q_{j(i)}(w)$
        \Comment{stop bridge, \Cref{eq:s-stop}; trained under $D_0$}
      \State $\bar p_i \gets [p_i(s^*),\,1-p_i(s^*)]$
    \EndIf
    \If{the whitespace mask is enabled and the realized token at row $i$ is entirely spaces or tabs}
      \State mask row $i$ out of the loss
        \Comment{whitespace mask; \Cref{sec:collapse}}
    \EndIf
  \EndFor
  \State $\mathcal{R} \gets$ unmasked content positions of $y$; \ $\mathcal{R}' \gets$ stop positions of $y$
  \State project each $p_i$ onto its explicit target cells plus the complement to obtain $\widetilde p_i$
  \State $\mathcal{L} \gets \mathcal{L} + \sum_{i \in \mathcal{R}} D_\beta(t_i \,\|\, \widetilde p_i)
    + \sum_{i \in \mathcal{R}'} D_0(b_i \,\|\, \bar p_i)$
    \Comment{\Cref{eq:loss}}
\EndFor
\State update $\theta$ with the gradient of $\mathcal{L}/Z$, where $Z$ is the
  batch's total student response-token count (masked and excluded positions
  included)
\end{algorithmic}
\end{algorithm}

\paragraph{Stopping-bridge gradient.}
Stop positions train under $D_0=\mathrm{KL}(b_i\,\|\,\bar p_i)$ rather than the generalized Jensen--Shannon divergence $D_\beta$ (\Cref{eq:jsd}) because only $D_0$ keeps a restoring gradient once the student's stop probability has collapsed. Take the gradient of the divergence in the student logit of the realized stop token $s^*$ as $p_i(s^*)\to 0$: under $D_\beta$ with $\beta\in(0,1)$ it vanishes, and a suppressed stop token receives no signal to recover, whereas under $D_0$ it approaches $-b_i(s^*)$, bounded away from zero whenever the teacher places stop mass at the aligned position (\Cref{eq:s-stop}). The $\phi$ stop-extension (\Cref{sec:objective}) already folds termination into the content target on a healthy run; the bridge adds the $D_0$ term to restore stopping in the suppressed regime.

\section{Exactness of the Byte-Prefix Marginal}
\label{app:proof}

Two questions underlie the closed form in \Cref{eq:s-scatter}: whether it preserves all teacher mass and whether it equals the byte-prefix marginal. We answer them in that order. After stating the byte-decoding and greedy-segmentation assumptions, we express the target as masses on a byte trie and prove that those masses telescope to one (\Cref{lem:telescope}). The scatter proposition then identifies this construction with the desired marginal (\Cref{prop:scatter}), before the final paragraph discusses the remaining gap between greedy longest-prefix matching and true BPE merge order.

\paragraph{Assumptions.}
\begin{itemize}
\item[(A1)] Every token $w$ decodes to a byte string $\bytes(w)$, and decoding a token sequence concatenates the byte strings losslessly (byte-level BPE).
\item[(A2)] At any byte position, the student tokenizer emits the longest token in $\Vs$ whose bytes prefix the remaining byte stream (greedy longest-prefix matching).
\end{itemize}
(A1) holds for the tokenizers we use by construction. (A2) idealizes BPE merge order; we argue its gap is second-order without proving a bound, and the end of this section reports that in position-by-position cross-checks we could not distinguish the two from exact tokenizer behavior.

\paragraph{Trie node-mass form.}
Fix an aligned student position $i$ and abbreviate $q = q_{j(i)}$, $\pi = \pi_{j(i)}$ (\Cref{eq:s-prefixmass}). Arrange $\Vs$ as a byte trie and let $\mathrm{tok}(\mathbf{c}) \in \Vs \cup \{\bot\}$ denote the longest student token whose bytes prefix $\mathbf{c}$'s path. The node-mass target assigns to each node $\mathbf{c}$ the mass $\rho(\mathbf{c})$ that ends exactly there:
\begin{equation}
\rho(\mathbf{c}) \;=\; \pi(\mathbf{c}) - \sum_{b=0}^{255} \pi(\mathbf{c} \oplus b),
\label{eq:residual}
\end{equation}
and credits it to $\mathrm{tok}(\mathbf{c})$, the deepest student-token ancestor on the path:
\begin{equation}
t_i(u) = \sum_{\mathbf{c}:\, \mathrm{tok}(\mathbf{c}) = u} \rho(\mathbf{c}),
\qquad
t_i(\varnothing) = \sum_{\mathbf{c}:\, \mathrm{tok}(\mathbf{c}) = \bot} \rho(\mathbf{c}).
\label{eq:trie-target}
\end{equation}

\begin{lemma}[Mass conservation]
\label[lemma]{lem:telescope}
For the assignment in \Cref{eq:residual,eq:trie-target}, $\sum_{u \in \Vs} t_i(u) + t_i(\varnothing) = 1$.
\end{lemma}

\begin{proof}
Every teacher token $w$ with $q(w)>0$ contributes $q(w)$ to $\pi(\mathbf{c})$ for exactly the prefixes $\mathbf{c} \preceq \bytes(w)$, which form one root-to-node path in the byte trie. Along this path, the node masses $\rho(\mathbf{c})$ of \Cref{eq:residual} telescope: each interior node's contribution $q(w)$ cancels between $\pi(\mathbf{c})$ and the child term $\pi(\mathbf{c}\oplus b)$ on the path, leaving $q(w)$ exactly once, at the node $\mathbf{c}=\bytes(w)$ (beyond which no prefix of $\bytes(w)$ continues). Hence $\sum_{\mathbf{c}} \rho(\mathbf{c}) = \sum_w q(w) = 1$, and \Cref{eq:trie-target} partitions this sum by $\mathrm{tok}(\mathbf{c}) \in \Vs \cup \{\bot\}$.
\end{proof}

\begin{proposition}[Scatter form of the marginal]
\label[proposition]{prop:scatter}
Assume (A1)--(A2). Let student position $i$ be aligned to teacher position $j(i)$, and suppose every student token inside the chunk is contained in a single teacher token (the student refines the teacher there). Then the byte-prefix marginal of \Cref{eq:s-def} equals the scatter form of \Cref{eq:s-scatter} for every $u \in \Vs \cup \{\varnothing\}$, and $\sum_{u} t_i(u) + t_i(\varnothing) = 1$.
\end{proposition}

\begin{proof}
Draw $w \sim q$. Because the student refines the teacher inside the chunk, the first student token emitted at position $i$ lies within $w^\rmt_{j(i)}$'s bytes, hence within $\bytes(w)$, and is determined by $w$ alone. Under (A2) it is the longest student token prefixing $\bytes(w)$, which is $\phi(w)$ of \Cref{eq:s-phi} by definition. The event in \Cref{eq:s-def} therefore coincides with $\{\phi(w)=u\}$, and taking probabilities under $q$ yields \Cref{eq:s-scatter}; the case $\phi(w)=\bot$ collects the teacher tokens no student token prefixes. Conservation follows either directly, since $\{\phi(w)=u\}_{u} \cup \{\phi(w)=\bot\}$ partitions $\Vt$, or from \Cref{lem:telescope}: the node mass $\rho(\bytes(w))$ is credited to $\mathrm{tok}(\bytes(w)) = \phi(w)$; hence \Cref{eq:trie-target} and \Cref{eq:s-scatter} agree token by token.
\end{proof}

\paragraph{The (A2) gap.}
Because BPE merge order occasionally emits a first token shorter than the longest matching prefix, (A2) is an idealization of counterfactual continuations. Two facts delimit its effect, though neither yields a proven bound. First, (A2) never changes \emph{which} token the student emitted (that token is observed), though it can perturb the target probability the marginal assigns to it and to its byte-prefix siblings. Second, the idealization only redistributes teacher mass among student tokens whose bytes prefix one another, and in position-by-position cross-checks we could not distinguish the two from exact tokenizer behavior.

\section{The Four Segmentation Relations, with a Worked Example}
\label{app:cases}

Within a chunk, the teacher-to-student token count takes one of four relations ($1{:}1$, $1{:}N$, $N{:}1$, $M{:}N$). \Cref{tab:cases} lists the five position cases they induce---$1{:}N$ contributes head and interior positions---with the target each receives and its guarantee; the paragraphs below give the constructions and a numeric walk-through. Ratios are written teacher\,:\,student.\footnote{The reference implementation labels the same cases in the opposite order, student\,:\,teacher.}

\begin{table}[t]
\centering
\setlength{\tabcolsep}{6pt}
\begin{tabular}{@{}llll@{}}
\toprule
Relation & Example & Target & Guarantee \\
\midrule
$1{:}1$ & \texttt{at} $\to$ \texttt{at} & \Cref{eq:s-scatter} & exact (\Cref{prop:scatter}) \\
$1{:}N$ head & \texttt{ate} $\to$ \texttt{at}\,$\cdot$\,\texttt{e} & \Cref{eq:s-scatter} & exact (\Cref{prop:scatter}) \\
$1{:}N$ int. & the \texttt{e} above & \Cref{eq:cond-interior} & exact \\
$N{:}1$ & \texttt{</think>} spans 3 & \Cref{eq:s-chain} & realized-path \\
$M{:}N$ & mid-start span & \Cref{eq:s-chain} & realized-path \\
\bottomrule
\end{tabular}
\caption{Target construction for each teacher-to-student segmentation relation. ``Realized-path'' targets are exact along the realized teacher segmentation and therefore lower-bound the full byte marginal for every spanning token, whether emitted or counterfactual (\Cref{sec:mismatch}).}
\label{tab:cases}
\end{table}

\paragraph{$1{:}1$ and $1{:}N$ head.}
When the student token ends at or inside the aligned teacher token, \Cref{prop:scatter} applies and the scatter of \Cref{eq:s-scatter} is the exact marginal.

\paragraph{$1{:}N$ interior.}
A student position inside a teacher token, with realized preceding bytes $\mathbf{r}_{\mathrm{pre}}$ already emitted since the teacher boundary, receives the conditional marginal: restrict the teacher distribution to the tokens whose bytes continue $\mathbf{r}_{\mathrm{pre}}$, renormalize by $\pi(\mathbf{r}_{\mathrm{pre}})$, and scatter the remaining bytes by the student's first token:
\begin{equation}
\begin{aligned}
t_i(u \mid \mathbf{r}_{\mathrm{pre}})
&= \frac{\sum_{w:\,\mathbf{r}_{\mathrm{pre}}\preceq\bytes(w),\;
\phi_{\mathbf{r}_{\mathrm{pre}}}(w)=u}q(w)}{\pi(\mathbf{r}_{\mathrm{pre}})},\\
t_i(\varnothing \mid \mathbf{r}_{\mathrm{pre}})
&= \frac{\sum_{w:\,\mathbf{r}_{\mathrm{pre}}\preceq\bytes(w),\;
\phi_{\mathbf{r}_{\mathrm{pre}}}(w)=\bot}q(w)}{\pi(\mathbf{r}_{\mathrm{pre}})}.
\end{aligned}
\label{eq:cond-interior}
\end{equation}
capped at the chunk boundary. This is exact: it is \Cref{eq:s-def} applied to the conditional teacher distribution, with every compatible outcome assigned to one longest student-token prefix or to $\varnothing$. It restates \Cref{eq:s-cond} for the $1{:}N$ interior case.

\paragraph{$N{:}1$ and $M{:}N$.}
Spanning student tokens receive the realized-path chain of \Cref{eq:s-chain}, applied as a mass-neutral reallocation from the head token's scatter cell to the spanning candidates, and each row thus remains a distribution over its explicit student-token cells and $\varnothing$. Positions that both start inside a teacher token and span past its end use the same chain conditioned on the observed $\mathbf{r}_{\mathrm{pre}}$; telescoping differences between nested candidates prevent prefix double-counting. Positions that neither start on a sync point nor reach one carry no exact target under the realized factorization and are excluded from the loss. On our training distribution these exclusions and all spanning corrections together account for well under one percent of response positions (measured per step: roughly $3{\times}10^4$ spanning rows against $6.6{\times}10^6$ aligned rows on the most dissimilar tokenizer pair), which is why the closed-form scatter dominates the cost. \Cref{tab:relstats} quantifies all four relations on each pair's realized generations.

\begin{table}[t]
\centering
\setlength{\tabcolsep}{6pt}
\begin{tabular}{@{}llrrr@{}}
\toprule
Relation & Example (teacher $\to$ student) & P1 & P2 & P3 \\
\midrule
$1{:}1$ & \texttt{the} $\to$ \texttt{the} & $99.65$ & $93.68$ & $94.75$ \\
$1{:}N$ & \texttt{10} $\to$ \texttt{1}\,$\cdot$\,\texttt{0} & $0.27$ & $6.24$ & $4.77$ \\
$N{:}1$ & \texttt{bit}\,$\cdot$\,\texttt{mask} $\to$ \texttt{bitmask} & $0.06$ & $0.06$ & $0.37$ \\
$M{:}N$ & \texttt{.f}\,$\cdot$\,\texttt{act}\,$\cdot$\,\texttt{orial} $\to$ \texttt{.factor}\,$\cdot$\,\texttt{ial} & $0.02$ & $0.02$ & $0.10$ \\
\bottomrule
\end{tabular}
\caption{Frequency of each segmentation relation among student positions in realized \method{} generations. P1/P2/P3 use $240$/$240$/$120$ responses and $5.4$/$4.1$/$3.1$M chunks, respectively; examples are the most frequent real instances in the measured corpora. P1 is almost entirely $1{:}1$, P2 is dominated by $1{:}N$ splits of GLM-Z1 tokens, especially digits, and P3 has the largest spanning share. Even on P3, the exact $1{:}1$ and $1{:}N$ cases cover $99.5\%$ of positions.}
\label{tab:relstats}
\end{table}

\paragraph{Worked example.}
Take teacher distribution $q = \{\texttt{at}{:}\,0.5,\ \texttt{ate}{:}\,0.2,\ \texttt{a}{:}\,0.2,\ \texttt{bit}{:}\,0.1\}$ at an aligned position, and a student vocabulary containing \texttt{a}, \texttt{at}, \texttt{ate}, \texttt{bit}. \Cref{tab:worked} computes the target twice: by trie node masses (\Cref{eq:residual,eq:trie-target}) and by the scatter (\Cref{eq:s-scatter}). Both give the same distribution and the mass telescopes to one, illustrating \Cref{lem:telescope} and \Cref{prop:scatter}.

\begin{table}[t]
\centering
\setlength{\tabcolsep}{6pt}
\begin{tabular}{@{}lcccc@{}}
\toprule
byte prefix $\mathbf{c}$ & \texttt{a} & \texttt{at} & \texttt{ate} & \texttt{bit} \\
\midrule
prefix mass $\pi(\mathbf{c})$ & 0.9 & 0.7 & 0.2 & 0.1 \\
node mass $\rho(\mathbf{c})$ & $0.9{-}0.7$ & $0.7{-}0.2$ & 0.2 & 0.1 \\
\midrule
$t(u)$ via trie & 0.2 & 0.5 & 0.2 & 0.1 \\
$t(u)$ via $\phi$-scatter & 0.2 & 0.5 & 0.2 & 0.1 \\
\bottomrule
\end{tabular}
\caption{Worked byte-prefix marginal. Trie node masses and the $\phi$-scatter agree for every student token, and $0.2{+}0.5{+}0.2{+}0.1=1$, confirming mass conservation. Each teacher token maps to a distinct student token here; \Cref{fig:overview}b shows the case in which several teacher tokens pool into one student token.}
\label{tab:worked}
\end{table}

For the spanning case, suppose the student emits \texttt{</think>} as one token where the teacher reads \texttt{</}\,$\cdot$\,\texttt{think}\,$\cdot$\,\texttt{>}, with realized-path probabilities $q_j(\texttt{</}) = 0.7$, $q_{j+1}(\texttt{think}) = 0.9$, and prefix mass $\pi_{j+2}(\texttt{>}) = 0.95$. \Cref{eq:s-chain} gives $t(\texttt{</think>}) = 0.7 \times 0.9 \times 0.95 \approx 0.60$, whereas the scatter of \Cref{eq:s-scatter} through $\phi(\texttt{</})$ would assign the spanning token essentially zero; the on-policy student would then be trained \emph{against} closing its own reasoning span. This failure is what the chain factorization of \Cref{sec:mismatch} exists to prevent.

\section{The Three Target Constructions on One Distribution}
\label{app:targets}

\Cref{fig:s-targets} holds one teacher distribution fixed so that only the target construction changes. Two teacher tokens, \texttt{then} and \texttt{they}, are absent from the student vocabulary. SimCT keeps only shared-surface tokens and renormalizes, leaving $44\%$ of the original teacher mass without supervision. ULD retains all rank slots but discards token identity by sorting both distributions. In contrast, \method{} routes \texttt{then} and \texttt{they} to \texttt{the}, their longest student-token byte prefix, while assigning genuinely unrepresentable mass to $\varnothing$. The comparison makes the three design requirements concrete: only the final construction preserves byte correspondence and total mass without restricting the student vocabulary.

\begin{figure*}[t]
\centering
\includegraphics[width=0.84\textwidth]{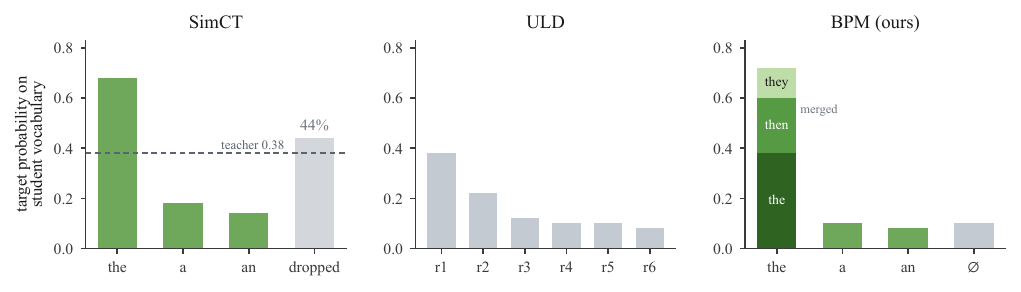}
\caption{Student-vocabulary targets constructed from the same teacher distribution over \texttt{the} ($.38$), \texttt{then} ($.22$), \texttt{they} ($.12$), \texttt{a} ($.10$), \texttt{an} ($.08$), and no-prefix mass ($.10$); \texttt{then} and \texttt{they} are absent from the student vocabulary, and colors track teacher-token mass. (a)~SimCT retains and renormalizes only shared-surface tokens, leaving $44\%$ of teacher mass unused. (b)~ULD aligns rank slots, so student-token identity is not preserved. (c)~\method{} routes the absent tokens to their longest student-token byte prefix \texttt{the}, routes unrepresentable mass to $\varnothing$, and preserves total mass.}
\label{fig:s-targets}
\end{figure*}

\section{Dispatch and Cost}
\label{app:impl}

The implementation separates tokenizer-dependent work from the repeated training path. A one-time precomputation builds the byte maps, trie, and routing table; during training, most positions reduce to one batched gather-scatter, while the rare interior and spanning cases receive targeted corrections. Content positions whose remaining bytes admit exactly one student continuation produce a point-mass target and train under $D_0$. The following paragraphs describe this dispatch, its fidelity checks and cost, and the evaluation conventions needed to reproduce the main results.

\paragraph{Implementation.}
We implement \method{} and every on-policy baseline arm on the slime training framework\footnote{\url{https://github.com/THUDM/slime}}. The teacher engine's sleep--wake scheduling follows KDFlow\footnote{\url{https://github.com/songmzhang/KDFlow}}. \Cref{tab:hparams} lists the configuration shared by the on-policy arms. SeqKD instead trains offline for three epochs at learning rate $1{\times}10^{-5}$ on teacher-generated responses, giving it no smaller training budget.

\begin{table}[t]
\centering
\small
\begin{tabular}{@{}ll@{}}
\toprule
\multicolumn{2}{@{}l}{\emph{Training}} \\
Student & Qwen3.5-2B (thinking mode) \\
Epochs & $2$ \\
Batch size & $256$ \\
Learning rate & $5\times10^{-7}$, constant \\
Default divergence & $D_\beta$ with $\beta=\tfrac12$ \\
Whitespace-row mask & on \\
\midrule
\multicolumn{2}{@{}l}{\emph{On-policy rollouts}} \\
Temperature & $1.0$ \\
top-$p$ / top-$k$ & $0.95$ / $20$ \\
Length cap & $27{,}648$ tokens \\
\midrule
\multicolumn{2}{@{}l}{\emph{Evaluation decoding}} \\
Temperature & $0.6$ \\
Responses per problem & $8$ (scored as avg@8 / pass@8) \\
\midrule
Hardware per run & one $8\times$H200 SXM node \\
\bottomrule
\end{tabular}
\caption{Training, rollout, and evaluation configuration shared by all on-policy arms. SeqKD is the exception: it fine-tunes offline for three epochs at learning rate $1{\times}10^{-5}$ on its teacher-generated corpus.}
\label{tab:hparams}
\end{table}

\paragraph{Hardware.}
Each training run occupies one $8\times$H200 SXM ($141$\,GB) node with the rollout engine, the teacher, and the trainer colocated. A few arms add data-parallel nodes of the same type for wall-clock efficiency, with the recipe otherwise unchanged.

\paragraph{Rollout decoding and stop sets.}
On-policy rollouts sample under the student's release decode settings (\Cref{tab:hparams}). The stop sets $\mathcal{S}^\rmt, \mathcal{S}^\rms$ are read from each model's declared stop tokens: the tokenizer's and the generation configuration's end-of-sequence entries and the chat template's turn-end token, since the turn-end token of several chat models differs from the tokenizer's \texttt{eos} field. Stop tokens are stripped from the byte stream before alignment. Keeping their literal bytes would work against the stopping target: the teacher assigns those characters near-zero probability, and the divergence would then pull the student away from stopping and toward the length limit.

\paragraph{Precomputation.}
Per tokenizer pair, once: byte maps for both vocabularies, the student byte trie, the map $\phi$ (one trie descent per teacher token, $O(\sum_{w}\lvert\bytes(w)\rvert)$ total), and the stop sets $\mathcal{S}^\rmt, \mathcal{S}^\rms$ from the generation configurations. The content part of $\phi$ is static; the stop extension $\phi(w)=s^*$ (\Cref{sec:objective}) is the one response-dependent entry, patched per response in $O(\lvert\mathcal{S}^\rmt\rvert)$.

\paragraph{Per step.}
Each response position is classified by its byte offsets into the relations of \Cref{tab:cases} and routed as follows: aligned and $1{:}N$-head positions to one batched gather-scatter of the teacher distributions through $\phi$ (\Cref{eq:s-scatter}); $1{:}N$-interior positions to the batched conditional marginal (\Cref{eq:cond-interior}); spanning positions to the chain correction (\Cref{eq:s-chain,eq:s-cond}), applied as a sparse update on top of the scatter; stop positions to the bridge target (\Cref{eq:s-stop}). The scatter dominates: it touches each aligned row once, with cost linear in the teacher vocabulary per row, and carries no byte-level traversal; only the rare interior-conditional and spanning rows perform bounded byte-trie descents, each linear in the token length. Classification itself is linear in the response length.

\paragraph{Fidelity.}
The vectorized scatter is checked against the pure byte-walk reference (\Cref{eq:residual,eq:trie-target}) position by position: agreement is at floating-point tolerance on $1{:}1$ and $1{:}N$ positions, and the spanning correction reproduces the chain value on realized tokens. The reference test suite also encodes the failure mode of \Cref{app:cases} as an assertion: a target built by first-token scatter alone must fail the spanning test, which guards against regressions that would silently reintroduce it.

\paragraph{Tokenizer-similarity metrics.}
The overlap coefficient of the Design paragraph of \Cref{sec:setup} is $\lvert \Vt \cap \Vs \rvert / \min(\lvert \Vt \rvert, \lvert \Vs \rvert)$, computed byte-exactly on token strings. Span Jaccard is the Jaccard similarity of the two tokenizers' sets of byte-offset segmentation spans on the training corpus. \Cref{fig:toksim} additionally reports vocabulary Jaccard, boundary $F_1$, and frequency-weighted coverage on the same $2{,}000$-prompt sample.

\begin{figure}[t]
\centering
\includegraphics[width=0.86\textwidth]{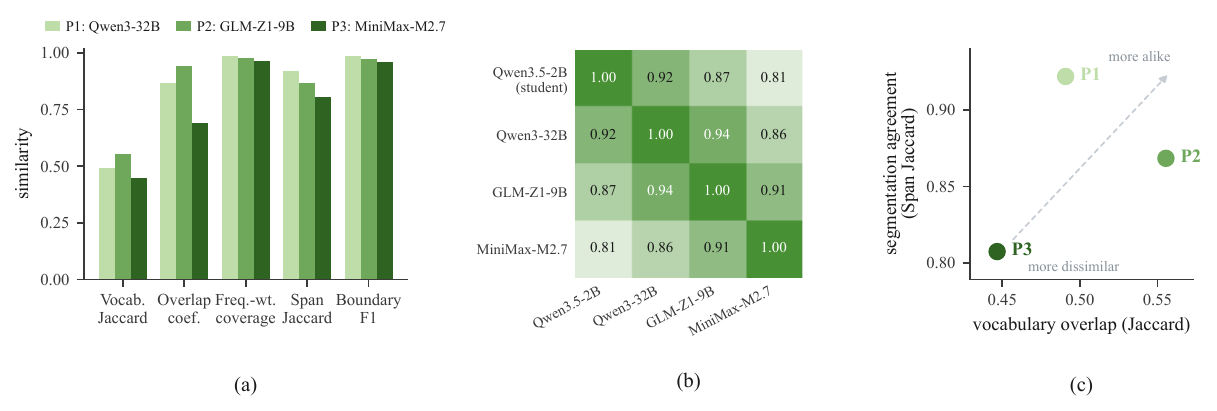}
\caption{Tokenizer similarity across the three teacher$\rightarrow$student pairs. (a)~Vocabulary Jaccard, overlap coefficient, frequency-weighted coverage, span Jaccard, and boundary $F_1$ on $2{,}000$ training prompts. (b)~Pairwise exact-span Jaccard for all four tokenizers on the same documents. (c)~P1--P3 in the vocabulary-versus-segmentation similarity plane.}
\label{fig:toksim}
\end{figure}

\paragraph{Decoding and metrics.}
Evaluation decoding uses temperature $0.6$ with $8$ samples per problem. \textbf{avg@8} is the mean correctness over the $8$ samples~\citep{chen2021codex}; \textbf{pass@8} is the fraction of problems solved by at least one sample. The February 2026 AIME and HMMT contests postdate every teacher's training cutoff. Headline margins carry $95\%$ paired-bootstrap intervals over problems ($10{,}000$ resamples), reflecting problem and sample variability in one regenerated evaluation per arm, but not training-run or checkpoint-selection uncertainty (\Cref{app:ci}).

\paragraph{Benchmark exclusions.}
For HumanEval+ we exclude one problem, HumanEval/32 (\texttt{find\_zero}). The exported official test special-cases this problem's oracle and unpacks the returned scalar root as an argument tuple. Every submission then fails with a \texttt{TypeError}, including the canonical solution, and the original oracle cannot be reconstructed from the export. We drop the problem and report HumanEval+ on the remaining $163$ problems, which run the unmodified official tests.

The full HumanEval+ set is $164$ problems, one excluded above. The LiveCodeBench slice is the most recent official window, $2025.01$--$2025.04$ ($182$ problems), each run against its full official test suite under the sampling protocol above. The TACO-test count ($283$ problems) and its construction are given in the extended-results section below.

\section{Task Prompts}
\label{app:prompts}

Training and every evaluation share one instruction template per task family, filled with the problem instance and then wrapped by the student model's own chat template with thinking enabled; responses are decoded up to $27{,}648$ tokens. Mathematics problems end with an \texttt{Answer:} line and code problems with a single fenced program, at training and test time alike. Braces denote fields substituted per instance.

\paragraph{Mathematical reasoning.}
The verifier reads the last \verb|\boxed{}| if present, else the string on the final \texttt{Answer:} line, and compares it to the reference with a symbolic checker. Integer answers (AIME) and the fractions, radicals, and closed forms that HMMT answers take are both handled: any expression equal to the reference counts, regardless of surface form.
\begin{promptbox}{Mathematical Reasoning Prompt (train = test)}
Solve the following math problem step by step. The last line of your response should be of the form Answer: $Answer (without quotes) where $Answer is the answer to the problem.

{problem}

Remember to put your answer on its own line after "Answer:".
\end{promptbox}

\paragraph{Code generation.}
The grader extracts the last syntactically valid fenced program from the response and executes it in a sandbox against the benchmark's official tests (stdin/stdout comparison or unit-test suites, per benchmark), at the official resource limits.
\begin{promptbox}{Code Generation Prompt (train = test)}
You are given a competitive programming problem. Think step by step, then provide a single complete Python program in a ```python code block. If a starter code / function signature is given, implement it; otherwise read from standard input and write to standard output.

{problem}
\end{promptbox}

\section{Extended Results}
\label{app:extra-results}

The following analyses unpack the aggregate results from four angles: benchmark-level divergence effects, method-agnostic distributional closeness, training stability and response length, and uncertainty in the headline margins.

\subsection{Per-Benchmark Scores for the Divergence Comparison}
\label{app:perds}

Across the three pairs, divergence choice moves six-benchmark avg@8 by three to five points (\Cref{tab:div}), a swing comparable to the main text's margins over the baselines. \Cref{tab:s-perds} traces each aggregate back to its six benchmark scores at the same selected checkpoints. Two patterns explain most of the movement. First, MATH-500 compresses stable arms into a $6.0$-point band ($78.8$--$84.8$ avg@8), whereas AIME 2026 spreads by as much as $12.5$ points within one block; contest mathematics and code therefore determine most of the divergence ranking. Second, FKL gains primarily on code. For P1, the three divergences lie within $1.6$ points in mathematics average ($46.1$--$47.7$) and within $1.3$ on AIME, but FKL leads the code average by $3.7$--$4.8$ points, with the widest gap on TACO ($13.3$ against $8.6$ and $10.5$). On P2, its lead spans both halves and is anchored by AIME ($35.4$ against $26.7$ and $28.3$).

P3 admits no single divergence ranking because the arms trade benchmarks. JSD posts the block's best mathematics average ($43.7$; AIME $32.5$ against FKL's $25.4$) and its worst HumanEval+ average ($40.6$ against $67.1$), yet HumanEval+ pass@8 remains nearly intact ($87.7$ against $92.0$). The arm still solves most problems in at least one sample but does so less consistently, indicating a consistency loss rather than a capability loss; \Cref{tab:s-ci} quantifies the same trade with paired intervals. The flagged RKL row shows a different failure: collapse appears unevenly across benchmarks. At its selected pre-collapse peak, RKL still posts the block's best AIME score ($37.9$), while MATH-500 has already fallen to $59.7$ against $81.7$--$81.8$ for its siblings. The larger $500$-problem set exposes degradation not yet sampled by the $30$-problem contest, and the code half trails throughout ($7.6$ on TACO).

The main text notes the \method{} margins are widest on the contests: the best arm beats each pair's strongest AIME 2026 baseline by $+1.7$ (P1, SimCT at $35.8$), $+12.1$ (P2, SeqKD at $23.3$), and $+7.1$ (P3, SeqKD at $18.3$). The $+12.1$ is wide partly because P2's strongest AIME baseline is anomalously weak on that benchmark; the pair itself is not harder, since AIME 2026 is the same test across all three pairs.

\begin{table*}[t]
\centering
\small
\setlength{\tabcolsep}{1.8pt}
\resizebox{\textwidth}{!}{%
\begin{tabular}{@{}l *{12}{c} cc@{}}
\toprule
 & \multicolumn{2}{c}{AIME 2026} & \multicolumn{2}{c}{HMMT 2026} & \multicolumn{2}{c}{MATH-500} & \multicolumn{2}{c}{HumanEval+} & \multicolumn{2}{c}{LiveCodeBench} & \multicolumn{2}{c}{TACO} & \multicolumn{2}{c}{All (6)} \\
\cmidrule(lr){2-3} \cmidrule(lr){4-5} \cmidrule(lr){6-7} \cmidrule(lr){8-9} \cmidrule(lr){10-11} \cmidrule(lr){12-13} \cmidrule(lr){14-15}
 & avg@8 & pass@8 & avg@8 & pass@8 & avg@8 & pass@8 & avg@8 & pass@8 & avg@8 & pass@8 & avg@8 & pass@8 & avg@8 & pass@8 \\
\midrule
\multicolumn{15}{@{}l}{\emph{P1 (Qwen3-32B teacher)}} \\
\quad FKL & 37.5 & 73.3 & 23.5 & 39.4 & 82.1 & 96.4 & 61.8 & 88.3 & 15.7 & 33.0 & 13.3 & 35.0 & 39.0 & 60.9 \\
\quad JSD & 37.1 & 66.7 & 19.3 & 36.4 & 82.0 & 94.6 & 57.4 & 89.6 & 10.6 & 29.1 & 8.6 & 30.4 & 35.8 & 57.8 \\
\quad RKL & 36.2 & 66.7 & 22.0 & 39.4 & 82.9 & 95.0 & 57.1 & 85.3 & 12.3 & 33.5 & 10.5 & 30.0 & 36.8 & 58.3 \\
\midrule
\multicolumn{15}{@{}l}{\emph{P2 (GLM-Z1-9B teacher)}} \\
\quad FKL & 35.4 & 63.3 & 21.2 & 36.4 & 84.8 & 95.4 & 68.8 & 91.4 & 22.3 & 33.5 & 16.5 & 36.4 & 41.5 & 59.4 \\
\quad JSD & 26.7 & 56.7 & 17.8 & 30.3 & 84.7 & 95.2 & 69.6 & 90.2 & 19.4 & 32.4 & 13.6 & 29.7 & 38.6 & 55.8 \\
\quad RKL & 28.3 & 50.0 & 20.5 & 33.3 & 78.8 & 94.8 & 65.3 & 89.0 & 16.1 & 30.2 & 11.8 & 30.4 & 36.8 & 54.6 \\
\midrule
\multicolumn{15}{@{}l}{\emph{P3 (MiniMax-M2.7 teacher)}} \\
\quad FKL & 25.4 & 46.7 & 18.6 & 39.4 & 81.8 & 95.6 & 67.1 & 92.0 & 21.8 & 34.1 & 17.4 & 34.6 & 38.7 & 57.1 \\
\quad JSD & 32.5 & 56.7 & 17.0 & 27.3 & 81.7 & 93.4 & 40.6 & 87.7 & 18.1 & 30.2 & 12.6 & 28.6 & 33.8 & 54.0 \\
\quad RKL & 37.9 & 56.7 & 17.8 & 36.4 & 59.7 & 86.2 & 32.5 & 75.5 & 15.6 & 26.9 & 7.6 & 22.6 & 28.5 & 50.7 \\
\bottomrule
\end{tabular}}%
\caption{Per-benchmark avg@8 and pass@8 (\%) for the three \method{} divergences on each pair. Each row uses the same selected checkpoint as \Cref{tab:div}, and All reproduces its six-benchmark aggregate. Per-benchmark results for the undistilled student and baselines appear in \Cref{tab:main}.}
\label{tab:s-perds}
\end{table*}

\subsection{Method-Agnostic Distributional Closeness}
\label{app:closeness}
We evaluate distributional closeness from four complementary directions, each computed from the two models alone and averaged over student rollouts on held-out mathematics. \emph{Teacher bits-per-byte} is the teacher's byte-normalized negative log-likelihood of student-generated text. \emph{Shared-mass ratio} measures teacher probability on shared-surface tokens at student-visited states, a concentration that prior on-policy-distillation work associates with success (\Cref{sec:setup}). \emph{Byte top-1 agreement} records how often the models' most likely next tokens begin with the same byte at shared boundaries. The reverse direction, \emph{student bits-per-byte on teacher text}, evaluates the student's negative log-likelihood of teacher-written solutions and is not directly optimized by any on-policy arm.

The direction of evaluation matters. Teacher bits-per-byte, shared mass, and byte agreement all reward a student the teacher finds predictable, so low-entropy repetition can improve them without faithful transfer. Student bits-per-byte on teacher text cannot obtain the same advantage from collapse and therefore provides the complementary fidelity check. On the two pairs whose baselines train to comparable depth, this reverse-direction measure ranks \method{} first among the distilled arms (\Cref{tab:s-universal}). Because a single-pair lead can fall within decoding noise, the claim rests on the repeated ordering rather than any one margin.

P3 requires a different interpretation. Its baselines are selected at step $24$, one sixth of the schedule and immediately before collapse; every reverse-direction value lies within $0.014$ bits of the base, and baseline gains are at most $0.003$. The column therefore measures distance from initialization rather than transfer on this pair. Teacher-facing metrics reverse their ordering: \method{} arms trained to step $99$ without collapse achieve the highest shared-mass and byte-agreement values, whereas the most degenerate baselines lead those columns on P1 and P2. The two directions exchange leaders precisely where the baselines remain closest to initialization, consistent with what the four metrics measure.

\paragraph{Baseline collapse timeline.}
The collapse timeline exposes the weakness of the three teacher-facing quantities. SimCT enters runaway repetition by iteration $24$ on MiniMax-M2.7 and after iteration $74$ on GLM-Z1-9B, while remaining stable on the student's own model family. As repetition lowers output entropy, the student becomes easier for the teacher to predict and can appear closer under teacher bits-per-byte or shared mass. Reverse-direction student bits-per-byte is what separates this predictable collapse from faithful transfer.

\begin{table}[t]
\centering
\setlength{\tabcolsep}{6pt}
\begin{tabular}{@{}lcccc@{}}
\toprule
Method & Tea.\ BpB\,$\downarrow$ & Shared\,$\uparrow$ & Byte agr.\,$\uparrow$ & Stu.\ BpB\,$\downarrow$ \\
\midrule
\multicolumn{5}{@{}l}{\emph{P1 (Qwen3-32B teacher)}} \\
Qwen3.5-2B (base) & $0.328$ & $0.955$ & $0.945$ & $0.251$ \\
\quad SimCT & $0.205$ & $0.995$ & $0.963$ & $0.267$ \\
\quad ULD & $0.162$ & $0.968$ & $0.975$ & $0.272$ \\
\quad GOLD & $0.172$ & $0.985$ & $0.972$ & $0.264$ \\
\quad \method{} (FKL) & $0.259$ & $0.993$ & $0.965$ & $0.263$ \\
\quad \method{} (JSD) & $0.229$ & $0.993$ & $0.965$ & $0.254$ \\
\quad \method{} (RKL) & $0.217$ & $0.992$ & $0.964$ & $0.254$ \\
\midrule
\multicolumn{5}{@{}l}{\emph{P2 (GLM-Z1-9B teacher)}} \\
Qwen3.5-2B (base) & $0.449$ & $0.930$ & $0.933$ & $0.242$ \\
\quad SimCT & $0.155$ & $0.952$ & $0.977$ & $0.248$ \\
\quad ULD & $0.231$ & $0.933$ & $0.963$ & $0.241$ \\
\quad GOLD & $0.226$ & $0.954$ & $0.969$ & $0.238$ \\
\quad \method{} (ours) & $0.250$ & $0.943$ & $0.964$ & $0.237$ \\
\midrule
\multicolumn{5}{@{}l}{\emph{P3 (MiniMax-M2.7 teacher)}} \\
Qwen3.5-2B (base) & $0.269$ & $0.893$ & $0.957$ & $0.282$ \\
\quad SimCT & $0.183$ & $0.916$ & $0.970$ & $0.279$ \\
\quad ULD & $0.201$ & $0.896$ & $0.966$ & $0.281$ \\
\quad GOLD & $0.195$ & $0.924$ & $0.972$ & $0.279$ \\
\quad \method{} (FKL) & $0.329$ & $0.928$ & $0.957$ & $0.295$ \\
\quad \method{} (JSD) & $0.295$ & $0.934$ & $0.963$ & $0.288$ \\
\quad \method{} (RKL) & $0.205$ & $0.934$ & $0.973$ & $0.282$ \\
\bottomrule
\end{tabular}
\caption{Distributional closeness across all three pairs, averaged over student rollouts on held-out mathematics at the checkpoint selected by six-benchmark avg@8. Unlabeled \method{} uses the default JSD ($\beta{=}\tfrac12$). Arrows indicate the nominal closer-to-teacher direction; student BpB evaluates teacher-written text and is not directly optimized by any on-policy arm. P3 baselines are selected at step $24$, before their collapse.}
\label{tab:s-universal}
\end{table}

\paragraph{P2 SimCT checkpoint recovery.}
The selected P2 SimCT checkpoint predates a checkpoint-retention incident and was recovered by replaying the identical recipe from initialization; the replayed and surviving checkpoints then entered the same checkpoint-selection rule as every other arm.

\paragraph{TACO-test benchmark construction.}
{\sloppy The competitive-programming test set in the main results is built from the official TACO test split (BAAI/TACO, dataset revision \texttt{d593ed0a}), restricted to the official \textsc{easy} and \textsc{medium} difficulty labels ($400$ of $1{,}000$ problems), with three further exclusions, each from official metadata or a disclosed screen: $75$ special-judge problems listed in the TACO repository's \texttt{output\_spj.jsonl} (our harness compares outputs exactly, so special-judge problems would grade as false negatives), $21$ problems with image assets or empty test suites, and $21$ problems whose statements near-duplicate our training pool under the same $n$-gram screens used for mix assembly (raw $13$-gram containment $\geq 0.7$, raw $8$-gram $\geq 0.6$, or NFKC $6$-gram $\geq 0.60$; all removals are Codeforces near-duplicates and are listed in the build manifest). The final set holds $283$ problems ($149$ \textsc{easy}, $134$ \textsc{medium}) across five judges, each with up to $40$ executable test cases. Because the training mix never touches the TACO test split (the pool is built from the train shards only), this set measures in-distribution transfer to unseen problems.\par}

\paragraph{Training-mix decontamination.}
\label{app:decontam}
The $20{,}000$-problem training mix is screened against all six evaluation benchmarks in two stages. The first stage drops verbatim leaks at pool construction, including problems that reappear in the February 2026 competitions, under the same $n$-gram containment thresholds used for the TACO-test screen above. The second stage re-screens at mix assembly, where no further removals were needed because the pool was already clean, and a final audit of all $20{,}000$ assembled prompts finds no residual leak.

\paragraph{Collapse-probe sample sizes.}
The whitespace-collapse measurements in \Cref{sec:collapse} report the following probe sizes, all on the controlled no-mask arm of the P2 pair: the indentation-target probe covers $957$ indentation rows and the off-realized-mass decomposition $477$; the margin probe spans $80$ programs with $1{,}799$ indentation positions ($22.5$ per program on average); the graded pass- and error-rate diagnostics use $128$ generations per checkpoint at the training-rollout temperature $1.0$.

\paragraph{Data-mixture ablation: construction and per-benchmark detail.}
The harder-math variant of \Cref{tab:datamix} replaces the mathematics half with harder competition-style problems and refreshes $1{,}250$ of the $10$k code problems. All three arms select step $149$ under the same rule, and \Cref{tab:s-datamixfull} expands their six-benchmark aggregates. FKL's $-2.1$ change is concentrated on the contests (AIME $35.4 \to 28.3$, HMMT $21.2 \to 18.2$), while its code average barely moves ($35.9 \to 35.6$), indicating that the easier original mathematics half aided contest transfer. JSD and RKL remain within $\pm 2.5$ per benchmark with no consistent direction; JSD even improves on AIME ($26.7 \to 29.2$). Response lengths are nearly unchanged between the mixes, with medians reported in \Cref{app:instability}.

The expanded-corpus variant grows the training set from $20$k to $35{,}527$ prompts ($1.78\times$): the mathematics half doubles ($+3$k from the same competition source, $+7$k from DeepMath-103K~\citep{he2025deepmath}) and the code half grows by $5.5$k (TACO, KodCode~\citep{xu2025kodcode}, and curated candidates), with zero contamination against all six evaluation benchmarks. Training runs the same protocol for $277$ steps. Scores plateau from step $99$ onward with no late collapse: checkpoints were evaluated every $25$ steps through the terminal save at step $274$, the six-benchmark average fluctuates within $\pm 1.5$ points of each arm's peak, and median AIME response lengths stay at $15$--$18$k tokens throughout, in contrast to the collapsing $20$k baselines of \Cref{tab:s-len}.

\begin{table*}[t]
\centering
\small
\setlength{\tabcolsep}{1.8pt}
\resizebox{\textwidth}{!}{%
\begin{tabular}{@{}l *{12}{c} cc@{}}
\toprule
 & \multicolumn{2}{c}{AIME 2026} & \multicolumn{2}{c}{HMMT 2026} & \multicolumn{2}{c}{MATH-500} & \multicolumn{2}{c}{HumanEval+} & \multicolumn{2}{c}{LiveCodeBench} & \multicolumn{2}{c}{TACO} & \multicolumn{2}{c}{All (6)} \\
\cmidrule(lr){2-3} \cmidrule(lr){4-5} \cmidrule(lr){6-7} \cmidrule(lr){8-9} \cmidrule(lr){10-11} \cmidrule(lr){12-13} \cmidrule(lr){14-15}
 & avg@8 & pass@8 & avg@8 & pass@8 & avg@8 & pass@8 & avg@8 & pass@8 & avg@8 & pass@8 & avg@8 & pass@8 & avg@8 & pass@8 \\
\midrule
\multicolumn{15}{@{}l}{\emph{Original mix (selected checkpoints, as in the divergence table)}} \\
\quad FKL & 35.4 & 63.3 & 21.2 & 36.4 & 84.8 & 95.4 & 68.8 & 91.4 & 22.3 & 33.5 & 16.5 & 36.4 & 41.5 & 59.4 \\
\quad JSD & 26.7 & 56.7 & 17.8 & 30.3 & 84.7 & 95.2 & 69.6 & 90.2 & 19.4 & 32.4 & 13.6 & 29.7 & 38.6 & 55.8 \\
\quad RKL & 28.3 & 50.0 & 20.5 & 33.3 & 78.8 & 94.8 & 65.3 & 89.0 & 16.1 & 30.2 & 11.8 & 30.4 & 36.8 & 54.6 \\
\midrule
\multicolumn{15}{@{}l}{\emph{Harder-math mix (all three select step $149$)}} \\
\quad FKL & 28.3 & 63.3 & 18.2 & 27.3 & 83.3 & 95.8 & 67.3 & 88.3 & 22.3 & 34.1 & 17.1 & 33.9 & 39.4 & 57.1 \\
\quad JSD & 29.2 & 60.0 & 19.3 & 33.3 & 82.9 & 94.0 & 66.9 & 87.7 & 20.6 & 34.6 & 14.7 & 32.5 & 38.9 & 57.0 \\
\quad RKL & 27.1 & 60.0 & 18.9 & 27.3 & 82.0 & 94.2 & 66.9 & 88.3 & 18.1 & 31.9 & 11.4 & 28.6 & 37.4 & 55.1 \\
\midrule
\multicolumn{15}{@{}l}{\emph{Expanded corpus (FKL and RKL select step $149$, JSD step $224$)}} \\
\quad FKL & 33.3 & 60.0 & 21.6 & 33.3 & 84.4 & 95.8 & 67.6 & 92.0 & 21.2 & 35.2 & 15.9 & 34.3 & 40.7 & 58.4 \\
\quad JSD & 30.8 & 63.3 & 19.7 & 33.3 & 81.1 & 94.0 & 68.9 & 90.8 & 19.4 & 34.6 & 13.6 & 33.6 & 38.9 & 58.3 \\
\quad RKL & 31.7 & 50.0 & 21.2 & 36.4 & 79.7 & 94.4 & 65.2 & 87.1 & 17.3 & 28.6 & 10.7 & 27.2 & 37.6 & 53.9 \\
\bottomrule
\end{tabular}}%
\caption{Per-benchmark expansion of the GLM-Z1-9B data ablation in \Cref{tab:datamix}: avg@8 and pass@8 (\%) under the original mix, the harder-math mix, and the expanded corpus. Each block header gives the selected checkpoint, and All averages the six benchmarks.}
\label{tab:s-datamixfull}
\end{table*}

\subsection{Divergence Instability and Length Behavior}
\label{app:instability}
Response length helps distinguish learned stopping from degeneration. We use it here to examine three observations from the main results: contest length-cap behavior after stop transfer, the late collapse of the pure mode-seeking arm, and the effect of the data-mixture ablation on generation length.

\paragraph{Contest length-cap behavior.}
\Cref{tab:s-len} reports, for every arm at its selected checkpoint, the median and mean response length and the fraction of samples that ran to the $27{,}648$-token cap on the two contests. The main text's headline reads off the first block: the undistilled base hits the AIME cap on $97\%$ of samples, and \method{}-FKL cuts that to $34\%$, the lowest of any distilled P1 arm (SimCT $50\%$, ULD $49\%$, GOLD $62\%$, SeqKD $61\%$). On GLM-Z1-9B the effect is strongest: \method{}-FKL stops early enough that only $8\%$ of AIME samples hit the cap. Two limits temper this. P1 ULD posts the shortest medians of its block and one of its lowest scores. Its brevity is degeneration rather than learned stopping. And on MiniMax-M2.7 the benefit does not arrive: every P3 arm, \method{} included, still caps at the median (cap fractions $62$--$90\%$ for the \method{} arms), and the teacher itself runs to the cap on $55\%$ of its HMMT samples. On this pair a verbose teacher's end-of-turn decision is a weaker lever. Where the median pins to the cap but the mean sits lower, more than half the samples cap while a minority stop early; we read the median as the typical response.

\paragraph{Late RKL collapse.}
On the MiniMax-M2.7 pair the RKL ($\beta{=}1$) arm peaks early and then collapses. After posting the block's best AIME 2026 score at step $99$, it degrades over the next fifty steps, and by step $155$ over $99\%$ of its contest responses run to the $27{,}648$-token cap, with repetition filling $56\%$ of their tokens. The FKL and $\beta{=}\tfrac12$ arms on the same pair stay stable past their selected checkpoints. The instability therefore needs pure mode-seeking and the most dissimilar pair together. Its selected step-$99$ row in \Cref{tab:s-len} already shows the highest \method{} cap fractions of the block, the collapse announcing itself in length before it reaches accuracy.

\paragraph{Data-mixture lengths.}
Response lengths under the two training mixes are close. Harder-math contest medians sit at $15$--$22$k tokens against the original mix's $15.6$--$23.7$k, and MATH-500 sits near $1.8$k under both. Neither mix moves the two code benchmarks, whose medians sit at the cap under both. Because the selected original-mix checkpoints are already below the cap on the contests, the mixture is not a length lever at these checkpoints. The truncation behavior in this section comes from the objective.

\begin{table}[t]
\centering
\small
\setlength{\tabcolsep}{3.4pt}
\begin{tabular}{@{}l rrr rrr@{}}
\toprule
 & \multicolumn{3}{c}{AIME 2026} & \multicolumn{3}{c}{HMMT 2026} \\
\cmidrule(lr){2-4} \cmidrule(lr){5-7}
Arm (step) & median & mean & cap\% & median & mean & cap\% \\
\midrule
Qwen3.5-2B (base) & $27648$ & $27551$ & $97$ & $27648$ & $27359$ & $95$ \\
\midrule
\multicolumn{7}{@{}l}{\emph{P1 (Qwen3-32B teacher; teacher: median $11588$/$18116$, cap $10$/$19\%$)}} \\
\quad SimCT (155) & $27648$ & $19762$ & $50$ & $27648$ & $22947$ & $66$ \\
\quad ULD (155) & $11684$ & $15554$ & $49$ & $4664$ & $11875$ & $34$ \\
\quad GOLD (99) & $27648$ & $19341$ & $62$ & $27648$ & $21086$ & $71$ \\
\quad SeqKD (191) & $27648$ & $20960$ & $61$ & $27648$ & $22548$ & $64$ \\
\quad \method{} (FKL, 155) & $21398$ & $19406$ & $34$ & $27640$ & $22011$ & $50$ \\
\quad \method{} (JSD, 155) & $23922$ & $19898$ & $44$ & $27648$ & $22452$ & $55$ \\
\quad \method{} (RKL, 155) & $21238$ & $19103$ & $38$ & $27648$ & $22116$ & $54$ \\
\midrule
\multicolumn{7}{@{}l}{\emph{P2 (GLM-Z1-9B teacher; teacher: median $6075$/$10712$, cap $0$/$3\%$)}} \\
\quad SimCT (49) & $27648$ & $22027$ & $71$ & $27648$ & $23942$ & $77$ \\
\quad ULD (49) & $15161$ & $15575$ & $22$ & $18007$ & $17716$ & $25$ \\
\quad GOLD (49) & $19144$ & $18132$ & $35$ & $26027$ & $21598$ & $48$ \\
\quad SeqKD (182) & $18122$ & $17849$ & $45$ & $27648$ & $20594$ & $54$ \\
\quad \method{} (FKL, 149) & $15584$ & $14687$ & $8$ & $18595$ & $17713$ & $14$ \\
\quad \method{} (JSD, 155) & $17568$ & $17206$ & $27$ & $21364$ & $19501$ & $36$ \\
\quad \method{} (RKL, 149) & $19428$ & $18421$ & $42$ & $23671$ & $20954$ & $46$ \\
\midrule
\multicolumn{7}{@{}l}{\emph{P3 (MiniMax-M2.7 teacher; teacher: median $21476$/$27648$, cap $41$/$55\%$)}} \\
\quad SimCT (24) & $27648$ & $27648$ & $100$ & $27648$ & $27448$ & $98$ \\
\quad ULD (24) & $27648$ & $25072$ & $82$ & $27648$ & $26005$ & $91$ \\
\quad GOLD (24) & $27648$ & $26708$ & $95$ & $27648$ & $26993$ & $96$ \\
\quad SeqKD (111) & $27648$ & $27170$ & $98$ & $27648$ & $26599$ & $95$ \\
\quad \method{} (FKL, 99) & $27648$ & $21632$ & $62$ & $27648$ & $23535$ & $75$ \\
\quad \method{} (JSD, 99) & $27648$ & $21499$ & $64$ & $27648$ & $23792$ & $78$ \\
\quad \method{} (RKL, 99) & $27648$ & $24710$ & $82$ & $27648$ & $25609$ & $90$ \\
\bottomrule
\end{tabular}
\caption{Contest response lengths at selected checkpoints: median and mean tokens, and cap\%, the share of eight samples per problem that reach the $27{,}648$-token limit. Pair headers give teacher statistics.}
\label{tab:s-len}
\end{table}

\subsection{Bootstrap Confidence Intervals for the Headline Margins}
\label{app:ci}

We attach uncertainty to each within-pair margin with a paired bootstrap over problems. For every pair, the selected checkpoints of the best \method{} arm and its comparator baseline are re-generated under the same protocol: $8$ samples per problem, temperature $0.6$, and identical graders. Each problem then contributes an avg@8 for both arms from one harness run. We resample problems with replacement $10{,}000$ times, recompute the margin, and report the $2.5$th and $97.5$th percentiles. For the Avg column, each resample draws problem indices separately within all six benchmarks and aggregates them jointly.

These problem-level intervals reflect variability in the fixed set of eight regenerated samples; they do not estimate variation across training runs or checkpoint selection, since each arm is trained once and its checkpoint is already chosen. Re-generation redraws all samples, so its point estimates differ from the original runs in \Cref{tab:main} by sampling noise. The largest observed per-benchmark deviation is $6.3$ points on the $30$-problem AIME 2026, which remains within one standard error. The same limitation applies to the closeness intervals reported above.

\Cref{tab:s-ci} presents four comparisons: the best arm against SimCT on P1 and against ULD on P2 and P3, plus JSD against ULD on P3 as an alternate-divergence check for the most dissimilar pair. On P1 and P3 the comparator is the pair's strongest baseline; on P2 it is the strongest distributional baseline, as SeqKD ranks higher there in \Cref{tab:main}.

\begin{table}[t]
\centering
\small
\setlength{\tabcolsep}{3pt}
\begin{tabular}{@{}lrrrr@{}}
\toprule
& \multicolumn{1}{c}{P1: FKL$-$SimCT} & \multicolumn{1}{c}{P2: FKL$-$ULD} & \multicolumn{1}{c}{P3: FKL$-$ULD} & \multicolumn{1}{c}{P3: JSD$-$ULD} \\
\midrule
AIME 2026   & $+0.8$ $[-5.0, +7.1]$ & $\mathbf{+11.3}$ $[+5.0, +17.9]$ & $\mathbf{+13.8}$ $[+6.3, +21.3]$ & $\mathbf{+10.4}$ $[+1.7, +19.6]$ \\
HMMT 2026   & $+2.7$ $[-3.0, +8.3]$ & $\mathbf{+4.9}$ $[+0.8, +9.9]$ & $\mathbf{+7.6}$ $[+1.9, +14.8]$ & $\mathbf{+6.4}$ $[+1.9, +11.7]$ \\
MATH-500     & $-0.1$ $[-1.6, +1.4]$ & $\mathbf{+2.8}$ $[+1.1, +4.5]$ & $\mathbf{+3.0}$ $[+0.9, +5.0]$ & $\mathbf{+4.3}$ $[+2.3, +6.3]$ \\
HumanEval+  & $\mathbf{+4.5}$ $[+0.8, +8.3]$ & $+2.4$ $[-1.9, +6.5]$ & $-0.7$ $[-5.1, +3.8]$ & $\mathbf{-23.9}$ $[-28.8, -18.9]$ \\
LCB         & $\mathbf{+12.2}$ $[+9.2, +15.4]$ & $\mathbf{+6.8}$ $[+4.2, +9.7]$ & $\mathbf{+3.7}$ $[+1.8, +5.7]$ & $+1.3$ $[-0.5, +3.2]$ \\
TACO        & $\mathbf{+10.1}$ $[+7.9, +12.3]$ & $\mathbf{+3.1}$ $[+0.9, +5.4]$ & $\mathbf{+4.2}$ $[+1.8, +6.6]$ & $+0.1$ $[-2.2, +2.3]$ \\
\midrule
Avg         & $\mathbf{+5.0}$ $[+3.3, +6.7]$ & $\mathbf{+5.2}$ $[+3.6, +6.9]$ & $\mathbf{+5.3}$ $[+3.4, +7.2]$ & $-0.2$ $[-2.2, +1.8]$ \\
\bottomrule
\end{tabular}
\caption{Paired-bootstrap avg@8 margins (percentage points) with $95\%$ intervals from $10{,}000$ problem-level resamples. Positive values favor the first method named; bold intervals exclude zero.}
\label{tab:s-ci}
\end{table}

\section{Case Studies}
\label{app:casestudy}

\newcommand{\tka}[1]{\colorbox{black!18}{\small\ttfamily\strut #1}}
\newcommand{\tkb}[1]{\colorbox{black!7}{\small\ttfamily\strut #1}}
\newcommand{\tsep}{\hspace{-0.4pt}{\color{white}\vrule width 0.8pt}\hspace{-0.4pt}}
\newcommand{\ska}[1]{\colorbox{bestbg}{\small\ttfamily\strut #1}}
\newcommand{\skb}[1]{\colorbox{blockbg}{\small\ttfamily\strut #1}}

Three case studies ground the paper's central mechanisms in raw material from the GLM-Z1-9B pair. Cases~1 and~2 read the release files of the two models. Case~3 draws on the collapse probe of \Cref{sec:collapse}: $128$ code generations per checkpoint at the rollout temperature, graded by the probe harness behind \Cref{fig:collapse}c.

\paragraph{Case 1: a spanning token at every reasoning boundary.}
Every thinking-mode response closes its reasoning span with the same bytes. The two release tokenizers read those bytes differently. Shading marks the token boundaries; each cell sits under exactly the bytes it covers:

\begin{center}
\setlength{\fboxsep}{0pt}%
\renewcommand{\arraystretch}{1.9}
\begin{tabular}{@{}r@{\qquad}l@{}}
\small Bytes &
  \small\ttfamily
  </think>\textbackslash n\textbackslash n\textasciigrave\textasciigrave\textasciigrave{}python\textbackslash n \\
\small Teacher (GLM-Z1-9B) &
  \tka{</}\tsep\tkb{think}\tsep\tka{>\textbackslash n\textbackslash n}\tsep\tkb{\textasciigrave\textasciigrave\textasciigrave}\tsep\tka{python}\tsep\tkb{\textbackslash n} \\
\small Student (Qwen3.5-2B) &
  \ska{</think>}\tsep\skb{\textbackslash n\textbackslash n}\tsep\ska{\textasciigrave\textasciigrave\textasciigrave}\tsep\skb{python}\tsep\ska{\textbackslash n} \\
\end{tabular}
\end{center}

The student token \texttt{</think>} covers the teacher tokens \texttt{</} and \texttt{think} entirely, plus the first byte of the third. This is the spanning case of \Cref{sec:mismatch} with residual prefix $\mathbf{r}=\texttt{>}$. The scatter of \Cref{eq:target} alone would give \texttt{</think>} near-zero mass at every reasoning boundary. The chain correction of \Cref{eq:chain} supplies the teacher's realized-path probability instead. On ordinary code and prose lines the two segmentations mostly coincide; the divergence concentrates at special tokens and whitespace.

\paragraph{Case 2: the teacher's stop token is not its declared \texttt{eos}.}
The stop extension and the bridge of \Cref{eq:stop} rest on the stop sets $\mathcal{S}^\rmt$ and $\mathcal{S}^\rms$. The release files of the pair show what a single declared field would miss:

\begin{center}
\small
\begin{tabular}{@{}lll@{}}
\toprule
 & Teacher (GLM-Z1-9B) & Student (Qwen3.5-2B) \\
\midrule
Declared \texttt{eos} field & \texttt{<|endoftext|>} & \texttt{<|im\_end|>} \\
Turn end in the chat template & \texttt{<|user|>} & \texttt{<|im\_end|>} \\
Collected stop set &
  \makecell[l]{\texttt{<|endoftext|>}, \texttt{<|user|>},\\
               \texttt{<|observation|>}} &
  \makecell[l]{\texttt{<|im\_end|>},\\ \texttt{<|endoftext|>}} \\
\bottomrule
\end{tabular}
\end{center}

The student's field is truthful: Qwen3.5 closes every turn with its declared \texttt{<|im\_end|>}. The teacher's field is not. The GLM-Z1 chat template appends no end token to the assistant turn; in a rendered conversation the assistant's content runs directly into the \texttt{<|user|>} that opens the next turn. The declared \texttt{<|endoftext|>} never appears there. A collector that trusts the \texttt{eos} field alone waits for a token the teacher does not emit. The teacher's stopping signal never arrives. The three-source collection of \Cref{sec:objective} recovers all three teacher ids, $\phi$ maps each of them to the student's realized stop token $s^*$, and the bridge of \Cref{eq:stop} sums the teacher's stopping mass over the full set.

\paragraph{Case 3: the whitespace collapse in one problem.}
The two \method{}-JSD arms of \Cref{fig:collapse} differ only in the whitespace-row mask. At step $149$, on one dynamic-programming problem from the probe, the unmasked arm generates:

\begin{casebox}{\method{}-JSD without the whitespace-row mask, step 149}
def robot_climb_ways(steps: int) -> int:
     if steps == 1:
          return 1
     if steps == 2:
          return 2
       current = 0
        prev2 = 1
        prev1 = 2
         for _ in range(steps - 2):
              current = prev1 + prev2
              prev2 = prev1
               prev1 = current
          return prev1
\end{casebox}

The algorithm is correct. The indentation drifts from five spaces to seven, eight, and nine, and the block fails with an \texttt{IndentationError} at line $6$. Across the probe, $108$ of $128$ blocks at this checkpoint fail the same way ($84\%$, the terminal point of \Cref{fig:collapse}c). The masked arm at the same step answers the same problem with:

\begin{casebox}{\method{}-JSD with the whitespace-row mask, step 149}
def robot_climb_ways(steps: int) -> int:
    """
    :param steps: An integer representing the number of steps to the top.
    :return: An integer representing the number of distinct ways to reach
             the top.
    """
    if steps <= 2:
        return steps if steps == 1 else 2
    a, b = 1, 2
    for _ in range(3, steps + 1):
        a, b = b, a + b
    return b
\end{casebox}

Four-space indentation holds throughout, the block compiles, and the tests pass. Across the probe, the masked arm produces zero indentation-broken blocks. Both excerpts are verbatim generations from the two checkpoints.

\end{document}